\RequirePackage{fix-cm}
\documentclass[smallextended]{svjour3}       
\smartqed  
\usepackage{graphicx}
\usepackage{lineno}
\usepackage{amssymb, graphicx, amsmath}
\usepackage{comment}

\usepackage[acronym]{glossaries}
\usepackage{rotating} 
\usepackage[show]{ed}

\newcommand{\eq}[1]{\begin{align*}#1\end{align*}}
\newcommand{\eqn}[1]{\begin{align}#1\end{align}}

\usepackage[authoryear]{natbib}

\begin{document}

\newacronym{auc}{AUC}{area under the ROC curve}
\newacronym{roc}{ROC}{receiver operating characteristic}
\newacronym{cindex}{C-index}{concordance index}
\newacronym{loo}{LOO}{leave-one-pair-out}
\newacronym{rmse}{RMSE}{root mean squared error}
\newacronym{mse}{MSE}{mean squared error}
\newacronym{mae}{MAE}{mean absolute error}
\newacronym{gip}{GIP}{Gaussian interaction profile}
\newacronym{R2}{$R^2$}{fraction of explained variance}
\newacronym{krr}{KRR}{kernel ridge regression}
\newacronym{seplr}{SEP-LR}{separable linear relational}
\newacronym{lpcs}{LPCS}{Labeled Point Cloud Superposition}
\newacronym{mcs}{MCS}{Maximum Common Subgraph}
\newacronym{cb}{CB}{CavBase}
\newacronym{fp}{FP}{Fingerprints}
\newacronym{sw}{SW}{Smith-Waterman}
\newacronym{pca}{PCA}{principal component analysis}
\newacronym{tss}{TSS}{two-step single hyperparameter}
\newacronym{svd}{SVD}{singular value decomposition}

\newtheorem{mycol}{Corollary}




\newcommand*\alternative[1]{\bar{#1}}  

\newcommand{\supsc}[1]{\ensuremath{^{\text{#1}}}}   
\newcommand{\subsc}[1]{\ensuremath{_{\text{#1}}}}   

\newcommand{\lsize}{n}  
\newcommand{\osize}{m}  
\newcommand{\qsize}{q}  
\newcommand{\testsize}{{n'}}
\newcommand{\odsize}{d}  
\newcommand{\qdsize}{r}
\newcommand{\rank}{R}
\newcommand{\partialu}{\mathbf{p}}
\newcommand{\partialv}{\mathbf{q}}
\newcommand{\realset}{\mathbb{R}}
\newcommand{\ispace}{\mathcal{X}}
\newcommand{\Vspace}{\mathcal{V}}
\newcommand{\loss}{\mathcal{L}}
\newcommand{\Uspace}{\mathcal{U}}
\newcommand{\anyspace}{\mathcal{X}}
\newcommand{\hypspace}{\mathcal{H}}
\newcommand{\labelmatrix}{\mathbf{Y}}
\newcommand{\predmatrix}{\mathbf{F}}
\newcommand{\objectu}{u}
\newcommand{\objectv}{v}
\newcommand{\objectx}{x}
\newcommand{\kernelf}{k}
\newcommand{\repru}{\boldsymbol{\phi}}
\newcommand{\reprv}{\boldsymbol{\psi}}
\newcommand{\gkernelf}{g}
\newcommand{\kkernelf}{\Gamma}
\newcommand{\hatmatrix}{\mathbf{H}}
\newcommand{\dkernelm}{\bm{K}}
\newcommand{\tkernelm}{\bm{G}}
\newcommand{\kkernelm}{\bm{\Gamma}}
\newcommand{\regparam}{\lambda}
\newcommand{\idmatrix}{\mathbb{I}}
\newcommand{\expectation}{\mathbb{E}}
\newcommand{\ones}{\mathbb{J}}
\newcommand{\trace}{\textnormal{tr}}
\newcommand{\transpose}{^\top}
\newcommand{\bm}[1]{\mathbf{#1}}
\newcommand{\ve}{\textnormal{vec}}
\newcommand{\mat}{\textnormal{mat}}
\newcommand{\diagv}{\textnormal{diag}_v}
\newcommand{\diagm}{\textnormal{diag}_m}
\newcommand{\LOO}{\mathrm{LOO}}
\newcommand{\predfun}{f}
\newcommand{\labely}{y}
\newcommand{\filterfun}{\varphi}
\newcommand{\uset}{U}
\newcommand{\database}{U'}
\newcommand{\vset}{V}
\newcommand{\edgeset}{E}
\newcommand{\edgespace}{\mathcal{E}}
\newcommand{\trainset}{S}
\newcommand{\testset}{T}
\newcommand{\labelvec}{\bm{y}}
\newcommand{\predvec}{\bm{f}}
\newcommand{\relation}{Q}
\newcommand{\topkset}{S^K_\objectv}

\newcommand{\argmax}[1]{\underset{#1}{\operatorname{arg}\,\operatorname{max}}\;}
\newcommand{\argmin}[1]{\underset{#1}{\operatorname{arg}\,\operatorname{min}}\;}

\title{A Comparative Study of Pairwise Learning Methods based on Kernel Ridge Regression
}


\author{Michiel Stock        \and
        Tapio Pahikkala  \and
        Antti Airola \and
        Bernard De Baets \and
        Willem Waegeman
}


\institute{KERMIT, Michiel Stock, Bernard De Baets and Willem Waegeman \at
              Department of Mathematical Modelling, Statistics and Bioinformatics, Ghent University, Belgium \\          \\
          Tapio Pahikkala and Antti Airola \at
              Department of Information Technology, University of Turku, Finland \\ \\
              contact: \email{michiel.stock@ugent.be} 
              }

\date{Received: date / Accepted: date}

\maketitle

\begin{abstract}
Many machine learning problems can be formulated as predicting labels for a pair of objects. Problems of that kind are often referred to as pairwise learning, dyadic prediction or network inference problems. During the last decade kernel methods have played a dominant role in pairwise learning. They still obtain a state-of-the-art predictive performance, but a theoretical analysis of their behavior has been underexplored in the machine learning literature. 

In this work we review and unify existing kernel-based algorithms that are commonly used in different pairwise learning settings, ranging from matrix filtering to zero-shot learning. To this end, we focus on closed-form efficient instantiations of Kronecker kernel ridge regression. We show that independent task kernel ridge regression, two-step kernel ridge regression and a linear matrix filter arise naturally as a special case of Kronecker kernel ridge regression, implying that all these methods implicitly minimize a squared loss. In addition, we analyze universality, consistency and spectral filtering properties. Our theoretical results provide valuable insights in assessing the advantages and limitations of existing pairwise learning methods. 
\end{abstract}
\section{Introduction to pairwise learning}\label{KMmultitask}
\subsection{Settings in pairwise learning}\label{sec:pwpredset}
Many real-world machine learning problems can naturally be represented as pairwise learning or dyadic prediction problems. In contrast to more traditional learning settings, the goal here consists of making predictions for pairs of objects $\objectu\in\Uspace$ and $\objectv\in\Vspace$, as elements of two universes $\Uspace$ and $\Vspace$. Such an ordered pair $(\objectu,\objectv)$ is often referred to as a dyad, and both elements in the dyad are usually equipped with a feature representation. In contrast to many statistical settings, these dyads are not independently and identically distributed, as the same objects tend to appear many times as part of different pairs.

Applications of pairwise learning often arise in the life sciences, such as predicting various types of interactions in all sorts of biological networks (e.g. drug-target networks, gene regulatory networks and species interaction networks). Similarly, pairwise learning methods are also used to extract novel relationships in social networks, such as author-citation networks. Other popular applications include recommender systems (predicting interactions between users and items) and information retrieval (predicting interactions between search queries and search results). 

Formally speaking, in pairwise learning one attempts to learn a function of the form $\predfun(\objectu,\objectv)$, i.e.\ a function to predict properties of two objects. Such functions are fitted using a set of $\lsize$ labeled examples: the training set $\trainset=\{(\objectu_h, \objectv_h, \labely_h) \mid h=1,\ldots,\lsize\}$. Further on, $\uset=\{\objectu_i \mid i=1,\dots,\osize\}$ and $\vset=\{\objectv_j \mid j=1,\ldots, \qsize\}$ will denote the sets of distinct objects of both types, later refered to as instances and tasks, respectively, in the training set with $\osize=\arrowvert\uset\arrowvert$ and $\qsize=\arrowvert\vset\arrowvert$.

Pairwise learning holds strong connections with many other machine learning settings. Especially a link with multi-task learning can be advocated, by calling the first object of a dyad an `instance' and the second object a `task'. 
The underlying idea for making the distinction between instances and tasks is that the feature description of the instances is often considered as more informative, while the feature description of the tasks is mainly used to steer learning in the right direction. Albeit less common in traditional multi-task learning formulations, feature representations for tasks play a crucial role in recent paradigms such as zero-shot learning -- see e.g.~\citet{Palatucci2009,Lampert2014}. 

The connection between pairwise learning and multi-task learning allows one to distinguish different prediction settings that are crucial in the context of the paper. Formally, four settings for predicting the label of the dyad $(\objectu,\objectv)$ can be distinguished in pairwise learning, based on whether testing objects are in-sample (appear in the training data) or out-of-sample (do not appear in the training data):
\begin{itemize}
\item {\bf Setting A}: Both $\objectu$ and $\objectv$ are observed during training, as parts of different dyads, but the label of the dyad $(\objectu,\objectv)$ must be predicted; 
\item {\bf Setting B}: Only $\objectv$ is known during training, while $\objectu$ is not observed in any training dyad, and the label of the dyad $(\objectu,\objectv)$ must be predicted; 
\item {\bf Setting C}: Only $\objectu$ is known during training, while $\objectv$ is not observed in any training dyad, and the label of the dyad $(\objectu,\objectv)$ must be predicted;  
\item {\bf Setting D}: Neither $\objectu$ nor $\objectv$ occur in any training dyad, and the label of the dyad $(\objectu,\objectv)$ must be predicted. 
\end{itemize}

Figure~\ref{fig:settings} visualizes data of the four settings graphically in four matrix representations. Setting A resembles a matrix completion or matrix filtering scenario, as typically encountered in collaborative filtering problems. In principle, feature representations are not needed if the structure of the matrix is exploited to generate predictions, but additional information might be helpful. Setting B resembles a classical multi-task learning scenario, where the columns represent instances and the rows tasks. For a predefined set of tasks, one aims for predicting the labels of novel instances. Setting C then considers the converse setting, where the instances are all known during training and some tasks are unobserved. This setting is in essence identical to Setting B, if one interchanges the notions of task and instance. Setting D is the most difficult prediction setting of all four. In the multi-task learning literature, this setting is known as zero-shot learning, as one aims for predicting the labels of tasks with zero training data.     

In pairwise learning, it is extremely important to distinguish these four prediction scenarios. Without bearing them in mind, one might select the wrong model for the given scenario or obtain an under- or overestimation of the generalization error. For example, a pairwise recommender system that can generalize well to new users might perform poorly for new items. In a large-scale meta-study about biological network identification, it was found that these concepts are vital to correctly evaluate pairwise learning models~\citep{Park2012}. Certain properties of different models discussed in this work only hold for certain settings.

\begin{figure}[t]
   \begin{center}
   \includegraphics[scale=0.8]{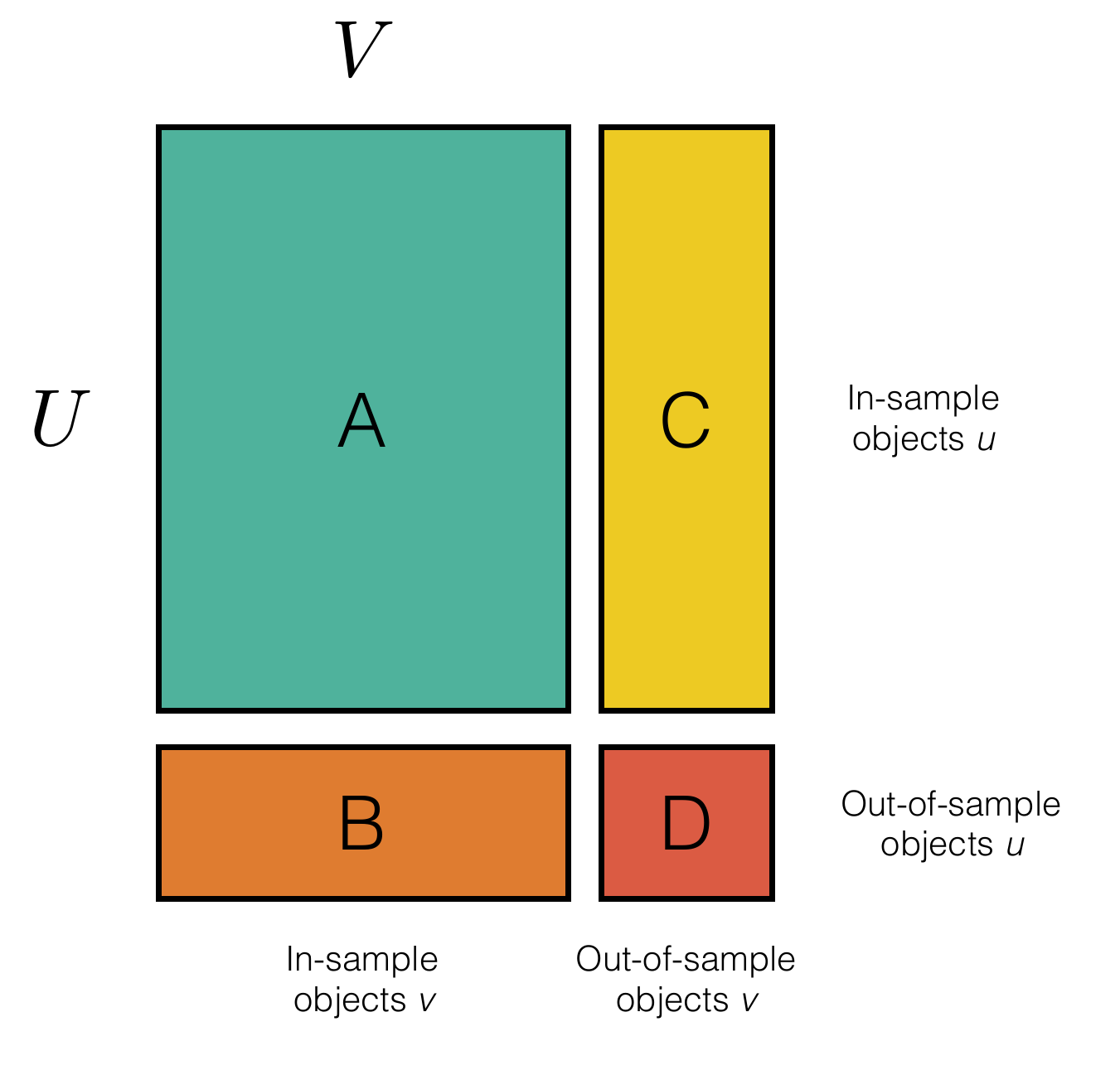} 
   \end{center}
   \caption{The different prediction settings in pairwise learning, depending on whether objects in a dyad occurred in the training set or not. Those four settings are further in this article always referred to as Setting A, B, C and D, respectively. }
   \label{fig:settings}
\end{figure}

\subsection{Kernel methods for pairwise learning with complete datasets}
During the last decade various types of methods for pairwise learning have been proposed in the literature.  Kernel methods in particular have been extensively used -- see e.g.\ \citet{Vert2005,Zaki2009,Huynh-Thu2010,VanLaarhoven2011,Cao2012,Liu2015a}. Especially in bio-informatics applications they have been popular, because biological entities are often more easy to represent in terms of similarity scores than feature representations~\citep{Ben-Hur2005,Shen2007,Vert2007}.\footnote{Recent advances in convolutional neural networks, however, have resulted in intriguing ways to generate representations for molecules~\citep{Duvenaud2015}, proteins~\citep{Jo2015} and nucleic acids~\citep{Alipanahi2015a}. Such feature representations, obtained by pretraining on large datasets, will likely be replace kernel methods in the future, at least to some extend.}

In this work, we will focus on kernel methods for pairwise learning. We believe that kernel methods have a number of appealing properties: 
\begin{itemize}
\item First, the existing methods that we analyze in this paper are general-purpose methods. They can be applied to a wide range of settings, including Settings~ A, B, C and D, and a wide range of application domains. More recent methods might outperform kernel methods in specific situations, but they are usually not applicable to Settings A, B, C and D at the same time, or they are mainly developed for specific application domains with very specific types of datasets, e.g.\ computer vision and text mining datasets. 
\item Second, the methods that we analyze often form an essential building block of more recent (and more complicated) methods. This is, for example, the case for zero-shot learning methods in computer vision. It is therefore important to provide a theoretical analysis of older methods, in order to gain a better understanding of more recent methods that are often black-box engineering approaches. More details on this aspect will be given in a related work section at the end of this article.   
\item Third, the methods that we analyze in this paper are still clear winners for specific scenarios. One of those scenarios is cross-validation in pairwise learning, for which kernel methods outperform other methods substantially w.r.t.\ computational scalability. Furthermore, scalable and exact algorithms can be derived to learn a model online or when the dataset is not complete (see Definition~\ref{def:complete} below). For more information on these aspects, we refer the reader to our complementary work~\citep{Stock2017exactiterative,Stock2017phd,Stock2018cvshortcuts}. 
\end{itemize}

These three reasons are the key motivations why it remains important to study kernel-based pairwise learning methods from a theoretical perspective. The key idea to extend kernel methods to pairwise learning is to construct so-called pairwise kernels, 
which measure the similarity between two dyads $(\objectu, \objectv)$ and $(\alternative{\objectu}, \alternative{\objectv})$. Kernels of that kind can be used in tandem with any conventional kernelized learning algorithm, such as support vector machines, \gls{krr} and kernel Fisher discriminant analysis.
In this article we will particularly focus on pairwise learning methods that are inspired by kernel ridge regression. Due to the algebraic properties of such methods, they are especially useful when analyzing so-called complete datasets in pairwise learning. 

\begin{definition}[Complete dataset] \label{def:complete}
A training set is called complete if it contains exactly one labeled example for every dyad $(\objectu,\objectv)\in\uset\times\vset$.
\end{definition}

If the label matrix contains only a few missing labels, matrix imputation methods can be applied to render the matrix complete~\citep{Mazumder2010,Stekhoven2012,Zachariah2012a}. Complete datasets, however, occur frequently, for example in biological networks such as drug-protein interactions or species interactions. Here, screenings or field studies generate a set of observed interactions, while interactions that or not observed are either interactions not occurring or false negatives~\citep{Schrynemackers2013,Jordano2016}. In such cases, the positive instances are labeled as 1 whereas the negatives are labeled 0. Theoretical work by~\citet{Elkan2008} has shown that models can still be learned from such datasets. Outside of biological network inference, complete datasets occur in recommender systems with implicit feedback, for example buying a book can be seen as a proxy for liking a book~\citep{Isinkaye2015}. Setting~A, i.e.~re-estimating labels, is still relevant for such datasets if the labels are noisy or contain false positives or false negatives. A pairwise learning model can be used to detect and curate such errors.

For a complete training set we introduce a further notation for the matrix of labels $\bm{Y}\in\mathbb{R}^{\osize\times\qsize}$, so that its rows are indexed by the objects in $\uset$ and the columns by the objects in $\vset$. Furthermore, we use $\bm{Y}_{i.}$, resp.\ $\bm{Y}_{.j}$, to denote the $i$-th row, resp.\ $j$-th column, of $\bm{Y}$. The vectorization of the matrix $\bm{Y}$ by stacking its columns in one long vector will be denoted $\bm{y}$.

\subsection{Scope and objectives of this paper}

The goal of this paper is to provide theoretical insights into the working of existing pairwise learning methods that are based on kernel ridge regression. To this end, we will focus on scenarios with complete training datasets, while analyzing the behavior for Settings~A, B, C and D. More specifically, we intend to provide an in-depth discussion of the following four methods: 

\begin{itemize}
\item Kronecker kernel ridge regression: adopting a least-squares formulation, this method is representative for many existing systems that are based on pairwise kernels. 
\item Two-step kernel ridge regression: this is a recent method that has some interesting properties such as simplicity and computational efficiency. The method has been independently proposed in \citet{Pahikkala2014} and \citet{Romera-paredes2015}. Also a variant of it exists, in which tree-based methods replace kernel ridge regression as base learners \citep{Schrynemackers2015}.  In a statistical context, similar models have been developed for structural equation modelling~\citep{Bollen1996,Bollen2004,Jung2013}.
\item Linear matrix filtering: this is a recently-proposed method that is able to provide predictions in Setting A without the need for object features, similar to collaborative filtering methods. Though simple, this linear filter was found to perform very well to predict interactions in a variety of species-species and protein-ligand interaction datasets~\citep{Stock2016a,Stock2017phd}. On these datasets it outperforms standard matrix factorization methods, and it is very tolerant to a large number of false negatives in the label matrices.  
\item Independent-task kernel ridge regression: this method serves as a baseline and a building block for some of the other methods. This approach resembles the traditional kernel ridge regression method, applied to each task (i.e.\ each column of $\bm{Y}$) separately.  When the method is applied to a single task, we will speak of single-task kernel ridge regression. 
\end{itemize}

We will review these four models in Section~2. They can all be represented using two positive semidefinite kernel functions, on for each type of objects, i.e.~$\kernelf:\Uspace\times\Uspace\rightarrow\realset$ and $\gkernelf:\Vspace\times\Vspace\rightarrow\realset$. These capture the similarity between two objects of the same types. We will deal with prediction functions of the form:
\eqn{
\predfun(\objectu, \objectv) = \sum_{i=1}^\osize \sum_{j=1}^\qsize a_{ij}  \kernelf(\objectu,\objectu_i) \gkernelf(\objectv,\objectv_j)\,,\label{eq:pairwisepredictionfunction}
}
with $\mathbf{A}=[a_{ij}] \in \realset^{\osize\times\qsize}$ the dual parameters. Such a model can, for instance, be obtained by the pairwise Kronecker kernel in a kernel-based learning algorithm such as support vector machines -- e.g.~\citet{Vert2007,Brunner2012}. In this work, we will limit ourselves to models where the dual parameters can be written as a linear combination of the label matrix:
\eqn{
\ve (\mathbf{A}) = \mathbf{B} \ve (\labelmatrix)\,.\label{eq:pairwiseparameter}
}
Here, $\mathbf{B}\in\realset^{\osize\qsize\times\osize\qsize}$ is a matrix constructed based on the training objects $\uset$ and $\vset$, the kernel functions and the learning algorithm, but \emph{not} the labels of the pairs. Similarly, the matrix containing the predictions $\predmatrix$ associated with the labels can be obtained by
\eqn{
\ve (\predmatrix) = \hatmatrix \ve (\labelmatrix)\,, \label{eq:pairwisehatmatrix}
}
where $\mathbf{H}\in\realset^{\osize\qsize\times\osize\qsize}$ is the so-called hat matrix which maps observations to predictions~\citep{hastie01statisticallearning}. Although $\mathbf{B}$ and $\mathbf{H}$ are huge matrices for problems of even modest sizes (i.e.~if $|\uset|$ and $|\vset|$ are in the order of thousands, these matrices have a cardinality of millions), for several methods the parameters and predictions can be computed efficiently. More specifically, the learning algorithms discussed in this work scale with the number of objects rather than the number of labels. 

The learning properties of the above four methods are theoretically analyzed in Section~\ref{theoreticalconsiderations}. In a first series of results, we establish equivalences via special kernels and algebraic operations. We discuss several links that are specific for Settings A, B, C or D. Figure~\ref{fig:overview} gives an overview of what the reader might expect to learn. In a second series of results we prove the universality of Kronecker product pairwise kernels, and we analyze the consistency of the algorithms that can be derived from such kernels. To this end, we provide a spectral interpretation of Kronecker and two-step kernel ridge regression. This will give further insights into the behavior of these methods.  

\begin{figure}[t]
   \begin{center}
   \includegraphics[scale=0.6]{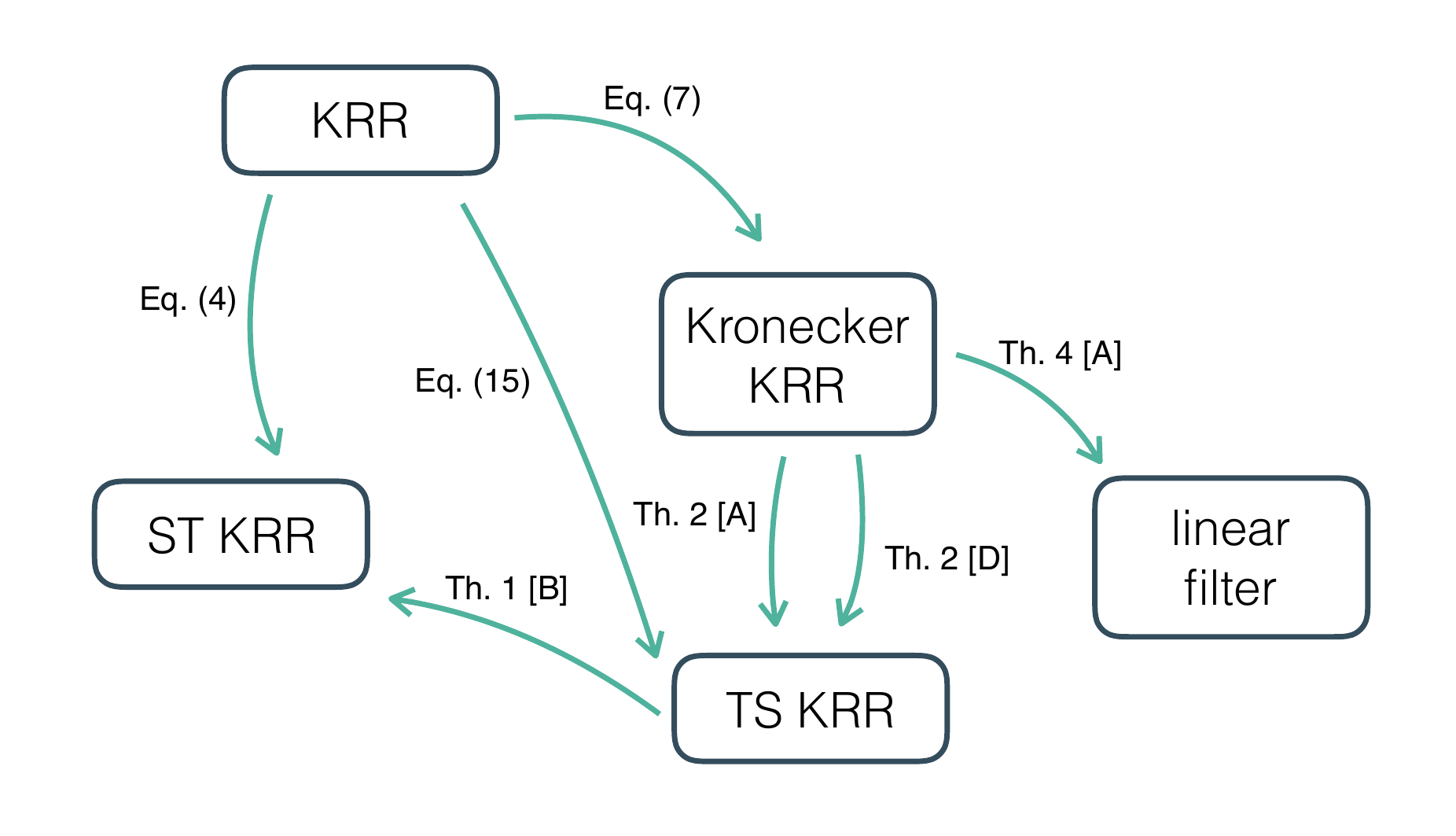} 
   \end{center}
   \caption{Overview of the different methods discussed in this work and their relation to one other: KRR = kernel ridge regression, ST KRR = single-task kernel ridge regression, TS KRR = two-step kernel ridge regression. The letters between brackets indicate the settings for which the theorem holds, as shown in Figure~\ref{fig:settings}.}
   \label{fig:overview}
\end{figure}

\section{Pairwise learning with methods based on kernel ridge regression} \label{dyadicprediciton}


In this section we formally review the four methods that were outlined in the introduction. We start by explaining a baseline multi-task learning formulation that will be needed to understand more complicated methods. We call this method independent-task kernel ridge regression, since it constructs independent models for the different tasks, i.e.\ the different columns of $\bm{Y}$. Subsequently, we elaborate on Kronecker kernel ridge regression as an instantiation of a method that employs pairwise kernels. In the last two paragraphs two-step kernel ridge regression and the linear matrix filter are reviewed.  In what follows we adopt a multi-task learning formulation, in which the objects of $\Uspace$ and $\Vspace$ are referred to as instances and tasks, respectively. 

\subsection{Independent-task kernel ridge regression}

Suppose that only features of objects of type~$\Uspace$ are available, but not of type $\Vspace$. Since there is no information available on how the tasks are related, a separate model for each task is trained. Let $\bm{Y}_{.j}\in\mathbb{R}^\osize$ be the labels of task $\objectv_j$ and $\kernelf(\cdot,\cdot)$ be a suitable kernel function what quantifies the similarity between the different instances. Since a separate and independent model is trained for each task, we will denote this setting as independent task (IT) kernel ridge regression. For each task $\objectv_j$, one would like to learn a function of the form
\eq{
\predfun_j^{\mathrm{IT}}(\objectu) = \sum_{i=1}^{\osize}
a_{ij}^\mathrm{IT} \kernelf(\objectu,\objectu_i) \,,
}
with $a_{ij}^\mathrm{IT}$ parameters that minimize a suitable objective function. In the case of \gls{krr}, this objective function is the squared loss with an $L_2$-complexity penalty. The parameters for the individual tasks using \gls{krr} can be found jointly by minimizing the following objective function~\citep{Wahba1990,Bishop2006}:
\eqn{
J(\bm{A}^\mathrm{IT}) = \trace [{(\dkernelm\bm{A}^\mathrm{IT} - \bm{Y})\transpose(\dkernelm\bm{A}^\mathrm{IT} - \bm{Y})}] + \regparam_\objectu \trace[{\bm{A}^\mathrm{IT}}\transpose\dkernelm\bm{A}^\mathrm{IT}]\, , \label{KRRpenalty}
}
with $\trace(\cdot)$ the trace, $\bm{A}^\mathrm{IT}=[a_{ij}^\mathrm{IT}]\in\mathbb{R}^{\osize\times\qsize}$ and $\dkernelm \in\mathbb{R}^{\osize\times\osize}$ the Gram matrix associated with the kernel function $\kernelf(\cdot,\cdot)$ for the instances and $\regparam_\objectu$ a regularization parameter. For simplicity, we assume the same regularization parameter $\regparam_\objectu$ for each task $\objectv$, though extensions to different penalties for different tasks are straightforward. This basic setting assumes no crosstalk between the tasks, as each model is fitted independently. The optimal coefficients that minimize Eq.\ (\ref{KRRpenalty}) can be found by solving the following linear system:
\eqn{\label{systemKRR}
\left(\dkernelm+\regparam_\objectu\idmatrix\right)\bm{A}^\mathrm{IT}=\bm{Y}\,.
}
Using the singular value decomposition of the Gram matrix, this system can be solved for any value of $\regparam_\objectu$ with a time complexity of $\mathcal{O}(\osize^3 +\osize^2 \qsize)$.

\subsection{Pairwise and Kronecker kernel ridge regression}\label{KMkroneckerkernel}

Suppose one does have prior knowledge about which tasks are more similar, quantified by a kernel function $\gkernelf(\cdot, \cdot)$ defined over the tasks. Several authors (see \citet{Alvarez2012,Baldassarre2012} and references therein) have extended \gls{krr} to incorporate task correlations via matrix-valued kernels. However, most of this literature concerns kernels for which the tasks are fixed at training time. An alternative approach, allowing for the generalization to new tasks more straightforwardly by means of such a task kernel, is to use a pairwise kernel $\kkernelf\left(\left(\objectu,\objectv\right),\left(\alternative{\objectu},\alternative{\objectv}\right)\right)$. Pairwise kernels provide a prediction function of the type
\eqn{\label{tensorprediciton}
\predfun(\objectu,\objectv) &= \sum_{h=1}^{\lsize}
\alpha_h\kkernelf\left(\left(\objectu,\objectv\right),\left({\objectu}_h,{\objectv}_h\right)\right) \,,
}
where $\boldsymbol{\alpha}= [\alpha_h]$ are parameters that minimize the same objective function as in (\ref{KRRpenalty}):
\eqn{\label{tikhonov}
J(\bm{\boldsymbol\alpha})=(\bm{\kkernelm}\bm{\boldsymbol\alpha} -\bm{y})\transpose(\bm{\kkernelm}\bm{\boldsymbol\alpha} -\bm{y})+\regparam\bm{\boldsymbol\alpha }\transpose\kkernelm\bm{\boldsymbol\alpha} \,,
}
with $\kkernelm$ the pairwise Gram matrix. The minimizer can also be found by solving a system of linear equations:
\eqn{\label{kronsystem}
\left(\kkernelm+\regparam\idmatrix\right)\bm{\boldsymbol\alpha} =\bm{y}\,.
}
The most commonly used pairwise kernel is the Kronecker product pairwise kernel \citep{Basilico2004,oyama2004using,Benhur2005,park2009pairwise,Hayashi2012,Bonilla2007,pahikkala2013conditional}. This kernel is defined as
\eqn{\label{pairwisekernel}
\kkernelf^\mathrm{KK}\left(\left(\objectu,\objectv\right),\left(\alternative{\objectu},\alternative{\objectv}\right)\right)=\kernelf\left(\objectu,\alternative{\objectu}\right)\gkernelf\left(\objectv,\alternative{\objectv}\right)\,,
}
a product of the data kernel $\kernelf(\cdot,\cdot)$ and the task kernel $\gkernelf(\cdot,\cdot)$. Many other variations of pairwise kernels have been considered to incorporate prior knowledge on the nature of the relations (e.g.~\citet{Vert2007,pahikkala2010intransitive,Waegeman2012,pahikkala2013conditional}) or for more efficient calculations in certain settings -- e.g.~\citet{Kashima2010}.

Let $\bm{\tkernelm}\in\mathbb{R}^{\qsize\times\qsize}$ be the Gram matrix for the tasks. Then, for a complete training set, the Gram matrix for the instance-task pairs is the Kronecker product $\kkernelm=\bm{\tkernelm}\otimes\bm{\dkernelm}$. Often it is infeasible to use this kernel directly due to its large size. The prediction function (\ref{tensorprediciton}) can be written as
\eqn{\label{pairwisedual}
\predfun^\mathrm{KK}(\objectu, \objectv)=\sum_{i=1}^\osize \sum_{j=1}^\qsize a_{ij}^\mathrm{KK} \kernelf(\objectu,\objectu_i) \gkernelf(\objectv,\objectv_j)\,.
}
The matrix $\mathbf{F}$ containing the predictions for the training data using a pairwise kernel can be obtained by a linear transformation of the training labels:
\eqn{
\ve(\mathbf{F})&= \kkernelm \ve(\bm{A}^\mathrm{KK})\\
&= \kkernelm\left(\kkernelm+\regparam\idmatrix\right)^{-1}\ve(\bm{Y})\label{PAIRWISEREEST}\\
&=\bm{H}^\kkernelf\ve(\bm{Y})\label{hatk} \,.
}
As a special case of Kronecker \gls{krr}, we also retrieve ordinary Kronecker kernel least-squares (OKKLS), when the objective function of Eq.~(\ref{tikhonov}) has no regularization term (i.e.~$\lambda=0$).

Several authors have pointed out that, while the size of the system in Eq.~(\ref{kronsystem}) is considerably large, its solutions for the Kronecker product kernel can be found efficiently via tensor algebraic optimization \citep{VanLoan2000,martin2006shiftedkron,Kashima2009,Raymond2010scalable,pahikkala2013conditional,Alvarez2012}. This is because the eigenvalue decomposition of a Kronecker product of two matrices can easily be computed from the eigenvalue decomposition of the individual matrices. The time complexity scales roughly with $\mathcal{O}(\osize^3+\qsize^3)$, which is required for computing the singular value decomposition of $\dkernelm$ and $\tkernelm$ (see Property~\ref{shiftedKronprop} in the appendix), but the complexities can be scaled down even further by using sparse kernel matrix approximation~\citep{Mahoney2011,Gittens2013}.

However, these computational short-cuts only concern the case in which the training set is complete. If some of the instance-task pairs in the training set are missing or if there are several occurrences of certain pairs, one has to resort, for example, to gradient-descent-based training approaches~\citep{park2009pairwise,pahikkala2013conditional,Kashima2009,Airola2017genvectric}. While the training can be accelerated via tensor algebraic optimization, such techniques still remain considerably slower than the approach based on eigenvalue decomposition.

\subsection{Two-step kernel ridge regression}
\begin{figure}[t]
   \begin{center}
   \includegraphics[scale=0.9]{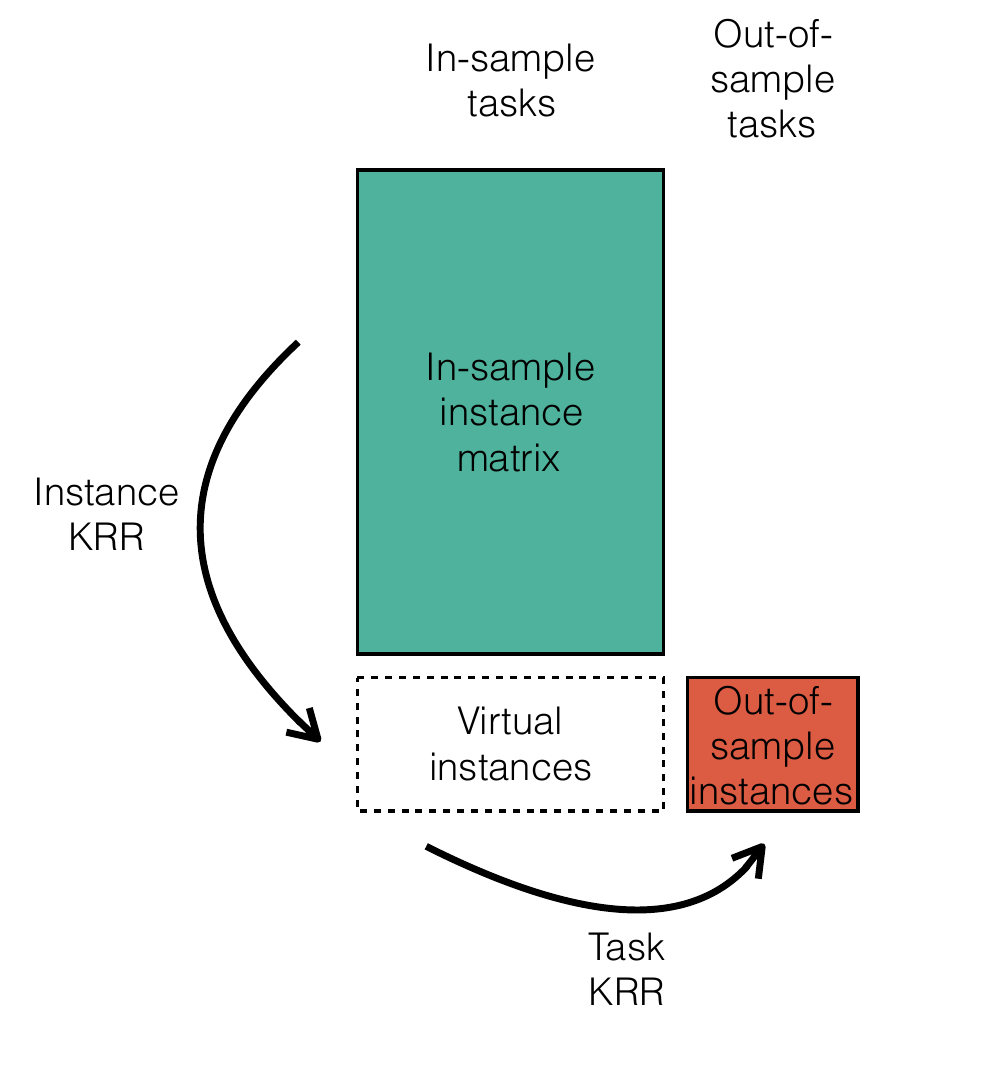} 
   \end{center}
   \caption{Principle of two-step \gls{krr}. In a first step, a virtual prediction is made for the out-of-sample tasks for new instances using a first \gls{krr} model. A second \gls{krr} model is trained using these data and this model is used to make predictions for new tasks.}
   \label{concept2SRLS}
\end{figure}

Clearly, independent-task ridge regression can generalize to new instances, but not to new tasks as no dependence between these tasks is encoded in the model. Kronecker \gls{krr}, on the other hand, can be used for all four prediction settings depicted in Figure~\ref{fig:settings}. But since our definition of `instances' and `tasks' is purely conventional, nothing is preventing us from building a model using the kernel function $\gkernelf(\cdot,\cdot)$ to generalize to new tasks for the same instances. By combining two ordinary \gls{krr}s, one for generalizing to new instances and one that generalizes to new tasks, one can indirectly predict for new dyads.

More formally, suppose one wants to make a prediction for the dyad $(\objectu,\objectv)$. Let $\bm{k}\in\mathbb{R}^{\osize}$ denote the vector of instance kernel evaluations between the instances in the training set and an instance in the test set, i.e.\ $\bm{k}(\objectu)=\left(\kernelf(\objectu,\objectu_1),\ldots,\kernelf(\objectu,\objectu_\osize)\right)\transpose$. Likewise, $\bm{g}\in\mathbb{R}^{\qsize}$ represents the vector of task kernel evaluations between the target task and the auxiliary tasks, i.e.
$\bm{g}(\objectv)=\left(\gkernelf(\objectv,\objectv_1),\ldots,\gkernelf(\objectv,\objectv_\qsize)\right)\transpose$. Based on the parameters found by solving Eq.~(\ref{systemKRR}), we can make a prediction for the new instance $\objectu$ for all the auxiliary tasks:
\eqn{\label{secondKRR}
\mathbf{\predfun}_\vset (\objectu) = \bm{k}\transpose \left(\dkernelm+\regparam_\objectu \idmatrix\right)^{-1} \bm{Y}  \,,
}
with $\regparam_\objectu$ the specific regularization parameter for the instances. This vector of predictions $\mathbf{\predfun}_\vset (\objectu)$ can be used as a set of labels in an intermediate step to train a second model for generalizing to new tasks for the same instance. Thus, using the task kernel and a regularization parameter for the tasks $\regparam_\objectv$, one obtains:
\eq{
\predfun^\mathrm{TS}(\objectu, \objectv) =  \bm{g}\transpose \left(\tkernelm+\regparam_\objectv \idmatrix\right)^{-1}  \mathbf{\predfun}_\vset (\objectu)\transpose \,,
}
or, by making use of Eq.~(\ref{secondKRR}), the prediction is given by
\eqn{\label{2SRLS}
\predfun^\mathrm{TS}(\objectu, \objectv) &= \bm{k}\transpose \left(\dkernelm+\regparam_\objectu \idmatrix\right)^{-1} \bm{Y}  \left(\tkernelm+\regparam_\objectv \idmatrix\right)^{-1} \bm{g} \\
&= \bm{k}\transpose \bm{A}^\mathrm{TS} \bm{g}\,,
}
with $\bm{A}^\mathrm{TS}$ the dual parameters. The concept of two-step \gls{krr} is illustrated in Figure~\ref{concept2SRLS}. Two-step \gls{krr} can be used for any of the prediction settings discussed in Section~\ref{sec:pwpredset}. Note that in practice there is no need to explicitly calculate $\mathbf{\predfun}_\vset$, nor does it matter if in the first step one uses a model for new tasks and in the second step for instances, or the other way around.

This model can be cast in a similar form as the pairwise prediction function of Eq.~(\ref{pairwisedual}) by making use of Property~\ref{pr:vectkron} in the appendix. Thus, for two-step \gls{krr} the parameters are given by
\eqn{
\bm{A}^\mathrm{TS} &=\left(\dkernelm+\regparam_\objectu \idmatrix\right)^{-1} \bm{Y}  \left(\tkernelm+\regparam_\objectv \idmatrix\right)^{-1}\,. \label{partwostepmatrix}
}
The time complexity for two-step \gls{krr} is the same as for Kronecker \gls{krr}: $\mathcal{O}(\osize^3+\qsize^3)$. The parameters can also be found by computing the eigenvalue decomposition of the two Gram matrices. Starting from these eigenvalue decompositions, it is possible to directly obtain the dual parameters for any values of the regularization hyperparameters $\lambda_\objectu$ and $\lambda_\objectv$. Because of its conceptual simplicity, it is quite straightforward to use two-step \gls{krr} for certain situations when the label matrix is not complete. The algebraic simplicity of two-step \gls{krr} can lead to some interesting algorithmic shortcuts for training and validating models. We refer to our other work for a theoretical and experimental overview~\citep{Stock2017phd}.

\subsection{Linear filter for matrices}\label{KMfeaturelesskernels}

Single-task \gls{krr} uses a feature description only for the objects $\objectu$, while Kronecker and two-step \gls{krr} incorporate feature descriptions of both objects $\objectu$ and $\objectv$. Is it possible to make predictions without any features at all? Obviously, this would only be possible for Setting~A, where both objects are known during training. The structure of the label matrix $\labelmatrix$, e.g.~being low rank, often contains enough information to successfully make predictions for this setting. In recommender systems, methods that do not take side features into account are often categorized as collaborative filtering methods~\citep{Su2009}. 

In order to use our framework, we have to construct some feature description, in the form of a kernel function. An object $\objectu$, resp.\ $\objectv$, can be described by the observed labels of the dyads that contain the object. In the context of item recommendation, this seems reasonable: users are described by the ratings they have given to items and, likewise, items are described by users' ratings. For example, Basilico and Hofman use a kernel based on the Pearson correlation of rating vectors of users to obtain a kernel description of users for collaborative filtering~\citep{Basilico2004a}. In bioinformatics, van Laarhoven and colleagues predict drug-target interactions using so-called Gaussian interaction profile kernels, i.e.\ the classical radial basis kernel applied to the corresponding row or column of the label matrix~\citep{VanLaarhoven2011a}. There is nothing inherently wrong with using the labels to construct feature descriptions or kernels for the object. One should only be cautious when taking a holdout set for model selection or model evaluation; the omitted labels should also be removed from the feature description to prevent overfitting.

Kernels that take observed labels into account, such as the Gaussian interaction profile kernel, are in theory quite powerful. As they can be used to learn nonlinear associations, they lead to more expressive models than matrix factorization. The advantage of using these kernels compared to other collaborative filtering techniques such as matrix factorization, $k$-nearest neighbors or restricted Boltzmann machines, is that side features can elegantly be incorporated into the model. To this end, one only has to combine the collaborative and content-based kernel matrices, for example, by computing a weighted sum or element-wise multiplication. 

Recently, a different method was proposed to make predictions without object features~\citep{Stock2016a}. This method makes a prediction for a couple $(\objectu_i,\objectv_j)$ by aggregating the observed value, the row- and column average and the total average of the label matrix. By analogy with an image filter, this method was called a linear filter (LF) for matrices. The prediction matrix (Eq.~(\ref{eq:pairwisehatmatrix})) is obtained as the following weighted average of averages:
\eqn{
\predmatrix_{ij}^\text{LF} = \alpha_1\labelmatrix_{ij} + \alpha_2\frac{1}{n}\sum_{k=1}^n\labelmatrix_{kj}+\alpha_3\frac{1}{m}\sum_{l=1}^m\labelmatrix_{il}+\alpha_4\frac{1}{nm}\sum_{k=1}^n\sum^m_{l=1}\labelmatrix_{kl} \label{eq:matrixfilter}\,,
}
where $(\alpha_1, \alpha_2, \alpha_3, \alpha_4)\in [0, 1]^4$. The first term is proportional to the label, while the last term is proportional to the average of all labels. The second (resp.\ third) term is proportional to the average label in the corresponding column (resp.\ row). The parameters $\alpha_1, \alpha_2, \alpha_3$ and $\alpha_4$ act as weighing coefficients. 

As mentioned in the introduction, this linear filter can outcompete standard methods such as matrix factorization and it was observed to be particularly tolerant to a large number of false negatives in the label matrices. An attractive property of the linear filter is that it is possible to derive a computational short-cut for \gls{loo} cross validation: 
\eqn{
\predmatrix_{ij}^{\text{LOO}} = \frac{\predmatrix_{ij} - \left(\alpha_1 + \frac{\alpha_2}{n} + \frac{\alpha_3}{m} + \frac{\alpha_4}{nm}\right)\labelmatrix_{ij}}{1-\left(\alpha_1 + \frac{\alpha_2}{n} + \frac{\alpha_3}{m} + \frac{\alpha_4}{nm}\right)}\,. \label{eq:filterlooshortcut}
}
This allows one to efficiently compute the prediction value $\predmatrix_{ij}^{\text{LOO}}$ using the label matrix \emph{except} for the value $\labelmatrix_{ij}$.

In Section~\ref{smoothertofilter} we will show that this linear filter is a special instance of Kronecker \gls{krr}. This filter can hence be written in the form of Eq.~(\ref{eq:pairwisepredictionfunction}) with the parameters obtained by solving a system of the form (\ref{eq:pairwiseparameter}). In practice, however, one would always prefer to work directly using Eq.~(\ref{eq:matrixfilter}). The parameters $\alpha_1, \alpha_2, \alpha_3$ and $\alpha_4$ can be set by means of leave-one-pair-out cross-validation using Eq.~(\ref{eq:filterlooshortcut}).

\section{Theoretical considerations}\label{theoreticalconsiderations}

In Subsections~\ref{equivalencetheorems} and~\ref{smoothertofilter} we show how the four methods of Section~2 are related via special kernels and algebraic equivalences. We establish several links that are specific for Setting~A, B, C or D. Therefore, each result is formulated as a theorem that indicates the setting to which it applies in its header. In Subsection~\ref{KMuniversality} the universality of the Kronecker product pairwise kernels is proven. This result provides a theoretical justification for the observation that Kronecker-based systems often obtain a very satisfactory performance in empirical studies. The universality is also used to prove the consistency of the methods that we analyze. This is done in Subsection~\ref{spectralInterpretation} via a spectral interpretation. In addition, this interpretation also allows us to illustrate that two-step kernel ridge regression adopts a special decomposable filter. 


\subsection{Equivalence between two-step and other kernel ridge regression methods}\label{equivalencetheorems}

The relation between two-step kernel ridge regression and independent-task ridge regression is given in the following theorem.

\begin{theorem}[Setting~B]\label{ITKRRTSeqiuivalence}
When the Gram matrix of the tasks $\bm{\tkernelm}$ is full rank and $\lambda_\objectv$ is set to zero, independent-task \gls{krr} and two-step \gls{krr} return the same predictions for any given training task:
\eq{
\predfun^\mathrm{IT}_j(\cdot)\equiv\predfun^\mathrm{TS}(\cdot, \objectv_j)\,.
}
\end{theorem}

\begin{proof}
The prediction for the independent-task \gls{krr} is given by:
\eq{
\predfun^\mathrm{IT}_j(\objectu) = [\bm{k}\transpose(\dkernelm+\regparam_\objectu\idmatrix)^{-1} \bm{Y}]_j \, .
}
For two-step \gls{krr}, it follows from Eq.~(\ref{2SRLS}) that
\eq{
\predfun^\mathrm{TS}_j(\objectu) &= [\bm{k}\transpose(\dkernelm+\regparam_\objectu\idmatrix)^{-1} \bm{Y}\tkernelm^{-1}\tkernelm]_j \\
 &=  [\bm{k}\transpose(\dkernelm+\regparam_\objectu\idmatrix)^{-1} \bm{Y}]_j\,.}
 
\end{proof}
When $\tkernelm$ is singular, the $\qsize$ outputs for the different tasks are projected on a lower-dimensional subspace by two-step \gls{krr}. This means that a dependence between the tasks is enforced, even when $\lambda_\objectv=0$.

The connection between two-step and Kronecker \gls{krr} is established by the following results.
\begin{theorem}[Setting~A]\label{KKTSequivalenceA}
Consider the following pairwise kernel matrix:
\eq{
\bm{\Xi} &= \bm{G}\otimes\bm{K}\left(\lambda_\objectu\lambda_\objectv\idmatrix\otimes\idmatrix+\lambda_\objectv\idmatrix\otimes\bm{K}+\lambda_\objectu\bm{G}\otimes\idmatrix\right)^{-1}\,.}
The predictions for the training data $\bm{F}$ using pairwise \gls{krr} (Eq.\ (\ref{PAIRWISEREEST})) with the above pairwise kernel and regularization parameter $\lambda=1$ correspond to those obtained with two-step \gls{krr} using the kernel matrices $\bm{K}$, $\bm{G}$ with respective regularization parameters $\lambda_\objectu$ and $\lambda_\objectv$.
\end{theorem}
\begin{proof}
We will formulate the corresponding empirical risk minimization of Eq.\ (\ref{tikhonov}) from the perspective of value regularization. Since Setting~A is an imputation setting, we directly search for the optimal predicted label matrix $\bm{F}$, rather than the optimal parameter matrix. Starting from the objective function for Kronecker \gls{krr}, the predictions for the training data are obtained through minimizing the following variational function:
\eqn{
J(\bm{F})&=\ve(\bm{F}-\bm{Y})\transpose\ve(\bm{F}-\bm{Y})+\ve(\bm{F})\transpose\bm{\Xi}^{-1}\ve(\bm{F}) \label{eq:TSERM}\\
&=\ve(\bm{F}-\bm{Y})\transpose\ve(\bm{F}-\bm{Y}) \nonumber \\
&\qquad+\ve(\bm{F})\transpose\left(\bm{G}\otimes\bm{K}\left(\lambda_\objectu\lambda_\objectv\idmatrix\otimes\idmatrix+\lambda_\objectv\idmatrix\otimes\bm{K}+\lambda_\objectu\bm{G}\otimes\idmatrix\right)^{-1}\right)^{-1}\ve(\bm{F})\nonumber \\
&=\ve(\bm{F}-\bm{Y})\transpose\ve(\bm{F}-\bm{Y}) \nonumber \\
&\qquad+\ve(\bm{F})\transpose\left(\bm{G}^{-1}\otimes\bm{K}^{-1}\left(\lambda_\objectu\lambda_\objectv\idmatrix\otimes\idmatrix+\lambda_\objectv\idmatrix\otimes\bm{K}+\lambda_\objectu\bm{G}\otimes\idmatrix\right)\right)\ve(\bm{F})\nonumber \\
&=\ve(\bm{F}-\bm{Y})\transpose\ve(\bm{F}-\bm{Y}) \nonumber \\
&\qquad+\ve(\bm{F})\transpose\left(\lambda_\objectu\lambda_\objectv\bm{G}^{-1}\otimes\bm{K}^{-1}+\lambda_\objectu\idmatrix\otimes\bm{K}^{-1}+\lambda_\objectv\bm{G}^{-1}\otimes\idmatrix\right)\ve(\bm{F})\nonumber \\ 
&=\trace((\bm{F}-\bm{Y})\transpose(\bm{F}-\bm{Y})
+\lambda_\objectu\lambda_\objectv\bm{F}\transpose\bm{K}^{-1}\bm{F}\bm{G}^{-1} \nonumber \\
&\qquad+\lambda_\objectu\bm{F}\transpose\bm{K}^{-1}\bm{F}+\lambda_\objectv\bm{F}\transpose\bm{F}\bm{G}^{-1})\,.\nonumber
}
The derivative with respect to $\bm{F}$ is given by:
\eq{
\frac{\partial J(\bm{F})}{\partial \bm{F}}&=2(\bm{F}-\bm{Y}+\lambda_\objectu\lambda_\objectv\bm{K}^{-1}\bm{F}\bm{G}^{-1}+\lambda_\objectu\bm{K}^{-1}\bm{F}+\lambda_\objectv\bm{F}\bm{G}^{-1})\\
&=2(\lambda_\objectu\bm{K}^{-1}+\idmatrix)\bm{F} + 2(\lambda_\objectu\bm{K}^{-1}+\idmatrix)(\lambda_v\bm{F}\bm{G}^{-1})-2\bm{Y}\\
&=2(\lambda_\objectu\bm{K}^{-1}+\idmatrix)\bm{F}(\lambda_\objectv\bm{G}^{-1}+\idmatrix)-2\bm{Y}\,.
}
Setting it to zero and solving with respect to $\bm{F}$ yields:
\eq{
\bm{F}&=(\lambda_\objectu\bm{K}^{-1}+\idmatrix)^{-1}\bm{Y}(\lambda_\objectv\bm{G}^{-1}+\idmatrix)^{-1}\\
&=\bm{K}(\bm{K}+\lambda_\objectu\idmatrix)^{-1}\bm{Y}(\bm{G}+\lambda_\objectv\idmatrix)^{-1}\bm{G}\,.
}
Comparing with Eq.~(\ref{partwostepmatrix}), we note that $\bm{F}=\bm{K}\bm{A}^\mathrm{TS}\bm{G}$, which proves the theorem.
\end{proof}

Here, we have assumed that $\dkernelm$ and $\tkernelm$ are invertible. Note that the kernel $\bm{\Xi}$ can always be obtained as long as $\dkernelm$ and $\tkernelm$ are positive semi-definite. The relevance of the above theorem is that it formulates two-step \gls{krr} as an empirical risk minimization problem for Setting~A (Eq.~(\ref{eq:TSERM})). It is important to note that the pairwise kernel matrix $\bm{\Xi}$ only appears in the regularization term of this variational problem. The loss function is only dependent on the predicted values $\bm{F}$ and the label matrix $\bm{Y}$. Using two-step \gls{krr} for Setting~A when dealing with incomplete data is thus well defined. The empirical risk minimization problem of Eq.~(\ref{eq:TSERM}) can be modified so that the squared loss only takes the observed dyads into account:  
\eqn{\label{eq:transdusctiveERMS}
J'(\bm{F})=\sum_{(\objectu, \objectv, y)\in S}(y - \predfun(\objectu, \objectv)))^2+\ve(\bm{F})\transpose\bm{\Xi}^{-1}\ve(\bm{F})\,,
}
with $S$ the training set of labeled dyads. In this case, one ends up with a transductive setting. This explains why Setting~A is the most easy setting to predict for, as in transductive learning one only has to predict for a finite number of dyads known during training, in contrast to inductive learning where the model has to make predictions for any new dyad, a harder problem~\citep{Chapelle2006}. See \citet{Rifkin2007,Johnson2008} for a more in-depth discussion.

Two-step and Kronecker \gls{krr} also coincide in an interesting way for Setting~D (e.g.\ the special case in which there is no labeled data available for the target task). This, in turn, will allow us to show the consistency of two-step \gls{krr} via its universal approximation and spectral regularization properties. The theorem below shows the relation between two-step \gls{krr} and ordinary Kronecker \gls{krr} for Setting~D.
\begin{theorem}[Setting~D] \label{KKTSequivalence}
Consider a setting with a complete training set. Let $\predfun^\mathrm{TS}(\cdot, \cdot)$ be a model trained with two-step \gls{krr} and $\predfun^\mathrm{OKKLS}(\cdot, \cdot)$ be a model trained with ordinary Kronecker kernel least-squares regression (OKKLS) using the following pairwise kernel function on $\Uspace\times\Vspace$:
\eqn{\label{twostepkernel}
\Upsilon\left(\left(\objectu,\objectv),(\alternative{\objectu},\alternative{\objectv}\right)\right)
=\left(\kernelf\left(\objectu,\alternative{\objectu}\right)
+\regparam_\objectu\delta\left(\objectu,\alternative{\objectu}\right)\right)
\left(\gkernelf\left(\objectv,\alternative{\objectv})
+\regparam_\objectv\delta\left(\objectv,\alternative{\objectv}\right)\right)\right)
}
where $\delta$ is the delta kernel whose value is 1 if the arguments are equal and 0 otherwise.
Then for making predictions for instances $\objectu\in\Uspace\setminus\uset$ and tasks $\objectv\in\Vspace\setminus\vset$ not seen in the training set, it holds that $\predfun^\mathrm{TS}(\objectu,\objectv)=\predfun^\mathrm{OKKLS}(\objectu, \objectv)$.
\end{theorem}
\begin{proof}
From Eq.~(\ref{2SRLS}) we have the following dual model for prediction:
\eq{
\predfun^\mathrm{TS}(\objectu, \objectv)=\sum_{i=1}^\osize \sum_{j=1}^\qsize a_{ij}^\mathrm{TS} \kernelf(\objectu,\objectu_i) \gkernelf(\objectv,\objectv_j)\,,
}
with $\bm{A}^\mathrm{TS}=[a_{ij}^\mathrm{TS}]$ the matrix of parameters. Similarly, the dual representation of the OKKLS (see Eq.~(\ref{pairwisedual})), using a parametrization $\bm{A}^\mathrm{OKKLS}=[a_{ij}^\mathrm{OKKLS}]$, is given by
\eq{
\predfun^\mathrm{OKKLS}(\objectu, \objectv)&=\sum_{i=1}^\osize \sum_{j=1}^\qsize a_{ij}^\mathrm{OKKLS} \Upsilon\left(\left(\objectu,\objectv),(\objectu_i,\objectv_j\right)\right) \\
&=\sum_{i=1}^\osize \sum_{j=1}^\qsize a_{ij}^\mathrm{OKKLS}(\kernelf(\objectu,{\objectu}_i)
+\regparam_\objectu\delta(\objectu,{\objectu}_i))
(\gkernelf(\objectv,{\objectv}_j)\\
&\qquad\qquad +\regparam_\objectv\delta(\objectv,{\objectv}_j)))\\
&=\sum_{i=1}^\osize \sum_{j=1}^\qsize a_{ij}^\mathrm{OKKLS}\kernelf(\objectu,\objectu_i) \gkernelf(\objectv,\objectv_j)\,.\\
}
In the last step we used the fact that $\objectu\neq\objectu_i$ and $\objectv\neq\objectv_j$ to drop the delta kernels. Hence, we need to show that $\bm{A}^\mathrm{TS}=\bm{A}^\mathrm{OKKLS}$.

By Eq.~(\ref{partwostepmatrix}) and denoting $\widetilde{\bm{G}}=\left(\bm{G}+\regparam\idmatrix\right)^{-1}$ and $\widetilde{\bm{K}}=\left(\bm{K}+\regparam\idmatrix\right)^{-1}$, we observe that the model parameters $\bm{A}^\mathrm{TS}$ of the two-step model can also be obtained in the following closed form:
\eqn{
\bm{A}^\mathrm{TS}&=\widetilde{\bm{K}}
\bm{Y}\widetilde{\bm{G}}
\,.
}
The kernel matrix of $\Upsilon$ for Setting~D can be expressed as:
$
\bm{\Upsilon} = \left(\bm{G}+\regparam_\objectv\idmatrix\right)\otimes\left(\bm{K}+\regparam_\objectu\idmatrix\right)\,.
$
The OKKLS problem with kernel $\Upsilon$ being
\begin{multline*}
\ve(\bm{A}^\mathrm{OKKLS})=\\
\qquad\argmin{\bm{A}\in\mathbb{R}^{\osize\times\qsize}}\left(\ve(\bm{Y})-\bm{\Upsilon}\ve(\bm{A})\right)\transpose\left(\ve(\bm{Y})-\bm{\Upsilon}\ve(\bm{A})\right)\,,
\end{multline*}
its minimizer can be expressed as
\eqn{\label{OKKLSparameters}
\ve(\bm{A}^\mathrm{OKKLS})&=\bm{\Upsilon}^{-1}\ve(\bm{Y})
=\left(\left(\bm{G}+\regparam_\objectv\idmatrix\right)^{-1}\otimes\left(\bm{K}+\regparam_\objectu\idmatrix\right)^{-1}\right)\ve(\bm{Y})\nonumber\\
&=\ve\left(\left(\bm{K}+\regparam_\objectu\idmatrix\right)^{-1}\bm{Y}\left(\bm{G}+\regparam_\objectv\idmatrix\right)^{-1}\right)
=\ve\left(\widetilde{\bm{K}}\bm{Y}\widetilde{\bm{G}}\right)\,.
}
Here we again make use of Property~\ref{pr:vectkron} in the appendix. From Eq.~(\ref{OKKLSparameters}) it then follows that $\bm{A}^\mathrm{TS}=\bm{A}^\mathrm{OKKLS}$, which proves the theorem.
\end{proof}

\subsection{Smoother kernels lead to the linear filter}\label{smoothertofilter}
Here, we will show that using Kronecker \gls{krr} in tandem with certain kernels results in the linear filter of Section~\ref{KMfeaturelesskernels}. When no good features of the objects are available, we propose to use different kernels, `agnostic' of the true objects:
\begin{align*}
\kernelf^\text{SM}\text(\objectu,\bar{\objectu}) &= 1 + \theta_\objectu \delta(\objectu,\bar{\objectu})\\
\gkernelf^\text{SM}(\objectv,\bar{\objectv}) &=  1 + \theta_\objectv \delta(\objectv,\bar{\objectu})\,,
\end{align*}
or, equivalently, as Gram matrices:
\begin{equation}
\dkernelm = \ones_{\osize\times\osize}+ \theta_\objectu \idmatrix_\osize \qquad \text{ and }\qquad\tkernelm = \ones_{\qsize\times\qsize}+ \theta_\objectv \idmatrix_\qsize\,.\label{eq:smootherkernels}
\end{equation}
Here, $\theta_\objectu$ and $\theta_\objectv$ are two hyperparameters of the kernels, $\ones_{\osize\times\osize}$ is an ${\osize\times\osize}$ matrix filled with ones and $\ones_{\qsize\times\qsize}$ is an ${\qsize\times\qsize}$ matrix filled with ones. We will call these kernels smoother kernels for reasons that will become clear. The rationale behind these kernels is quite simple: a kernel that is the identity matrix would imply that all objects are unique and independent; there is no similarity between them. Using the all-ones matrix on the other hand encodes all objects being exactly the same; no distinction between two objects can be made. Hence, the kernels of (\ref{eq:smootherkernels}) represent a trade-off between all objects being similar (first part) and all objects being unique (second part). This is controlled by the hyperparameters $\theta_\objectu$ and $\theta_\objectv$.

Using these kernels in the Kronecker-based models has an interesting interpretation: the predictions can be written as a weighted sum of averages.

\begin{theorem}[Smoother kernels]
Predictions using Kronecker \gls{krr} for Setting~A using the Gram matrices (\ref{eq:smootherkernels}) are of the form:
\eq{
\predfun(\objectu_i, \objectv_j) = \alpha_1\labelmatrix_{ij} + \alpha_2 \frac{1}{\qsize}\sum_{l=1}^\qsize\labelmatrix_{il} + \alpha_3 \frac{1}{\osize}\sum_{k=1}^\osize\labelmatrix_{kj} + \alpha_4 \frac{1}{\osize\qsize} \sum^\osize_{k=1}\sum^\qsize_{l=1}\labelmatrix_{kl}\,,
} 
with $(\alpha_1, \alpha_2, \alpha_3, \alpha_4)\in\realset^4$.
\end{theorem} 
\begin{proof}
For Setting~A, the hat matrix of Eq.~(\ref{eq:pairwisehatmatrix}) transforms the label matrix in the prediction matrix. The hat matrix $\hatmatrix^\text{SM}$ for Kronecker \gls{krr} using these smoother kernels can be obtained by applying Eq.~(\ref{hatk}):
\eqn{\label{eq:smootherform}
\hatmatrix^\text{SM} = a_1 \idmatrix_\qsize\otimes\idmatrix_\osize + a_2 \ones_\qsize\otimes\idmatrix_\osize + a_3 \idmatrix_\qsize\otimes\ones_\osize + a_4 \ones_\qsize\otimes\ones_\osize\,.
}
This can easily be seen because the pairwise Gram matrix will be of the form (\ref{eq:smootherform}) and Properties~\ref{pr:smoothmult} and \ref{pr:smoothinv} state that multiplying or inverting matrices of the form (\ref{eq:smootherform}) results in a matrix of the same form. From these properties, it follows that the hat matrix will also be of this form. 

The prediction for dyad $(\objectu_i, \objectv_j)$ is given by $[\hatmatrix^\text{SM} \ve(\labelmatrix)]_{j\qsize+i}$. Using the relation between the Kronecker product and the vectorization operation, each term of (\ref{eq:smootherform}) can be rewritten using Property~\ref{pr:vectkron} as follows
\begin{align*}
a_1 [(\idmatrix_\qsize\otimes\idmatrix_\osize) \ve(\labelmatrix)]_{j\osize+i} =a_1 [\idmatrix_\osize \labelmatrix\idmatrix_\qsize]_{ij}&=a_1\labelmatrix_{ij}\\
a_2 [(\ones_{\qsize\times\qsize}\otimes\idmatrix_\osize) \ve(\labelmatrix)]_{j\osize+i} =a_2 [\idmatrix_\osize \labelmatrix\ones_{\qsize\times\qsize}]_{ij} &= a_2\sum_{l=1}^\qsize\labelmatrix_{il}\\
a_3 [(\idmatrix_\qsize\otimes\ones_{\osize\times\osize}) \ve(\labelmatrix)]_{j\osize+i} =a_3 [\ones_{\osize\times\osize} \labelmatrix\idmatrix_\qsize]_{ij}&=a_3\sum_{k=1}^\osize\labelmatrix_{kj}\\
a_4 [(\ones_{\qsize\times\qsize}\otimes\ones_{\osize\times\osize}) \ve(\labelmatrix)]_{j\osize+i} =a_4 [\ones_{\osize\times\osize} \labelmatrix\ones_{\qsize\times\qsize}]_{ij}&=a_4\sum_{l=1}^\qsize\sum_{k=1}^\osize\labelmatrix_{kl}\,,
\end{align*}
which proves the theorem.

\end{proof}

The smoother kernel is thus quite restrictive in the type of models that can be learned. It can only exploit the fact that some rows or columns have a larger average value (e.g.~in item recommendation, some items in collaborative filtering have a high average rating, independent for the user). Nevertheless, it can lead to good baseline predictions for Setting~A and is particularly useful for small datasets with no side-features, such as species interaction networks.
\subsection{Universality of the Kronecker product pairwise kernel}\label{KMuniversality}
Here we consider the universal approximation properties of Kronecker \gls{krr} and, by Theorems~\ref{KKTSequivalenceA} and \ref{KKTSequivalence}, of two-step \gls{krr}. This is a necessary step in showing the consistency of the latter method. First, recall the concept of universal kernel functions.
\begin{definition} \label{kerneluniversalitydef}
{\bf~\citep{Steinwart2002consistency}} A continuous kernel $\kernelf(\cdot,\cdot)$ on a compact metric space $\anyspace$ (i.e.\ $\anyspace$ is closed and bounded) is called universal if the reproducing kernel Hilbert space (RKHS) induced by $\kernelf(\cdot,\cdot)$ is dense in $C(\anyspace)$, where $C(\anyspace)$ is the space of all continuous functions $\predfun : \anyspace \rightarrow \mathbb{R}$.
\end{definition}
The universality property indicates that the hypothesis space induced by a universal kernel can approximate any continuous function on the input space $\anyspace$ to be learned arbitrarily well, given that the available set of training data is large and representative enough, and the learning algorithm can efficiently find this approximation from the hypothesis space~\citep{Steinwart2002consistency}. In other words, the learning algorithm is consistent in the sense that, informally put, the hypothesis learned by it gets closer to the function to be learned while the size of the training set gets larger. The consistency properties of two-step \gls{krr} are considered in more detail in Subsection~\ref{spectralInterpretation}.

\newcommand{\tempfirstfun}{a}
\newcommand{\tempsecondfun}{b}
\newcommand{\temppairfun}{t}

Next, we consider the universality of the Kronecker product pairwise kernel. The following result is a straightforward modification of some of the existing results in the literature (e.g.~\citet{Waegeman2012}), but it is presented here for self-containedness. This theorem is mainly related to Setting~D, while it also covers the other settings as special cases.
\begin{theorem}\label{th:universality}
The kernel $\kkernelf^\mathrm{KK}((\cdot,\cdot),(\cdot,\cdot))$ on $\Uspace\times\Vspace$ defined in Eq.\ (\ref{pairwisekernel}) is universal if the instance kernel $\kernelf(\cdot,\cdot)$ on $\Uspace$ and the task kernel $\gkernelf(\cdot,\cdot)$ on $\Vspace$ are both universal.
\end{theorem}
\begin{proof}
Let us define
\begin{equation}\label{funckron}
\begin{array}{l}
\mathcal{A}\otimes\mathcal{B}
=\left\{\temppairfun\mid t(\objectu,\objectv)=\tempfirstfun(\objectu)\tempsecondfun(\objectv),\tempfirstfun\in \mathcal{A},\tempsecondfun\in \mathcal{B}\right\}
\end{array}
\end{equation}
for compact metric spaces $\Uspace$ and $\Vspace$ and sets of functions $\mathcal{A}\subset C(\Uspace)$ and $\mathcal{B}\subset C(\Vspace)$. We observe that the RKHS of the kernel $\Gamma$ can be written as $\mathcal{H}(\kernelf)\otimes\mathcal{H}(\gkernelf)$, where $\mathcal{H}(\kernelf)$ and $\mathcal{H}(\gkernelf)$ are the RKHS of the kernels $\kernelf(\cdot,\cdot)$ and $\gkernelf(\cdot,\cdot)$, respectively.

Let $\epsilon>0$ and let $\temppairfun\in C(\Uspace)\otimes C(\Vspace)$ be an arbitrary function that, according to Eq.~(\ref{funckron}), can be written  as $\temppairfun(\objectu,\objectv)=\tempfirstfun(\objectu)\tempsecondfun(\objectv)$, where $\tempfirstfun\in C(\Uspace)$ and $\tempsecondfun\in C(\Vspace)$. By definition of the universality property, $\mathcal{H}(\kernelf)$ and $\mathcal{H}(\gkernelf)$ are dense in $C(\Uspace)$ and $C(\Vspace)$, respectively. Therefore, there exist functions $\alternative{\tempfirstfun}\in\mathcal{H}(\kernelf)$ and $\alternative{\tempsecondfun}\in\mathcal{H}(\gkernelf)$ such that
\[
\max_{\objectu\in\Uspace}\left\arrowvert \alternative{\tempfirstfun}(\objectu)-\tempfirstfun(\objectu)\right\arrowvert\leq \alternative{\epsilon},\qquad \max_{\objectv\in\Vspace}\left\arrowvert \alternative{\tempsecondfun}(\objectv)-\tempsecondfun(\objectv)\right\arrowvert\leq \alternative{\epsilon} \,,
\]
where $\alternative{\epsilon}$ is a constant for which it holds that
\[
\max_{\objectu\in\Uspace,\objectv\in\Vspace}\left\{\left\arrowvert \alternative{\epsilon} \, \tempfirstfun(\objectu)\right\arrowvert+\left\arrowvert\alternative{\epsilon} \,\tempsecondfun(\objectv)\right\arrowvert+\alternative{\epsilon}^2\right\}\leq \epsilon \,.
\]
Note that, according to the extreme value theorem, the maximum exists due to the compactness of $\Uspace$ and $\Vspace$ and the continuity of the functions $\tempfirstfun(\cdot)$ and $\tempsecondfun(\cdot)$. Now we have
\[
\begin{array}{l}
\displaystyle
\max_{\objectu\in\Uspace,\objectv\in\Vspace}\left\arrowvert t(\objectu,\objectv)-\alternative{\tempfirstfun}(\objectu)\alternative{\tempsecondfun}(\objectv)\right\arrowvert\\
\displaystyle
\leq\max_{\objectu\in\Uspace,\objectv\in\Vspace}\left\arrowvert t(\objectu,\objectv)-\tempfirstfun(\objectu)\tempsecondfun(\objectv)\right\arrowvert+\left\arrowvert \alternative{\epsilon} \,\tempfirstfun(\objectu)\right\arrowvert+\left\arrowvert\alternative{\epsilon}\,\tempsecondfun(\objectv)\right\arrowvert+\alternative{\epsilon}^2\\
\displaystyle
=\max_{\objectu\in\Uspace,\objectv\in\Vspace}\left\arrowvert \alternative{\epsilon} \,\tempfirstfun(\objectu)\right\arrowvert+\left\arrowvert\alternative{\epsilon} \,\tempsecondfun(\objectv)\right\arrowvert+\alternative{\epsilon}^2\leq \epsilon,
\end{array}
\]
which confirms the density of $\mathcal{H}(\kernelf)\otimes\mathcal{H}(\gkernelf)$ in $C(\Uspace)\otimes C(\Vspace)$.

The space $\Uspace\times\Vspace$ is compact if both $\Uspace$ and $\Vspace$ are compact according to Tikhonov's theorem. It is straightforward to see that $C(\Uspace)\otimes C(\Vspace)$ is a subalgebra of $C(\Uspace\times\Vspace)$, it separates points in $\Uspace\times\Vspace$, it vanishes at no point of $C(\Uspace\times\Vspace)$, and it is therefore dense in $C(\Uspace\times\Vspace)$ due to the Stone-Weierstra{\ss} theorem. Thus, $\mathcal{H}(\kernelf)\otimes\mathcal{H}(\gkernelf)$ is also dense in $C(\Uspace\times\Vspace)$, and $\Gamma$ is a universal kernel on $\Uspace\times\Vspace$.
\end{proof}

\subsection{Spectral interpretation and consistency}\label{spectralInterpretation}

In this subsection we will study the difference between independent-task, Kronecker and two-step \gls{krr} from the point of view of spectral regularization. The above shown universal approximation properties of this kernel are also connected to the consistency properties of two-step \gls{krr}, as is elaborated in more detail in this subsection.

Learning by spectral regularization originates from the theory of ill-posed problems. This paradigm is well studied in domains such as image analysis~\citep{Bertero1998} and, more recently, in machine learning -- e.g.~\citet{LoGerfo2008}. Here, one wants to find the parameters $\bm{\boldsymbol\alpha}$ of the data-generating process given a set of noisy measurements $\labelvec$ such that
\eqn{\label{illposedproblem}
\kkernelm\bm{\boldsymbol\alpha} \approx \labelvec\,,
}
with $\kkernelm$ a Gram matrix with eigenvalue decomposition $\kkernelm = \bm{W}\bm{\Lambda}\bm{W}\transpose$. At first glance, one can find the parameters $\bm{\boldsymbol\alpha}$ by inverting $\kkernelm$:
\eq{
\bm{\boldsymbol\alpha} & = \kkernelm^{-1}\labelvec \\
 &=  \bm{W}\bm{\Lambda}^{-1}\bm{W}\transpose\labelvec\,.
}
If $\kkernelm$ has small eigenvalues, the inverse becomes highly unstable: small changes in the feature description of the label vector will lead to huge changes in $\bm{\boldsymbol\alpha}$. Spectral regularization deals with this problem by generalizing the inverse by a so-called filter function to make solving Eq.~(\ref{illposedproblem}) well-posed. The following definition of a spectral filter-based regularizer is standard in the machine learning literature (see e.g.~\citet{LoGerfo2008} and references therein).
Note that we assume $\kkernelf((\cdot,\cdot),(\cdot,\cdot))$ being bounded with $\kappa>0$ such that $\sup_{\bm{x}\in\ispace}\sqrt{\kkernelf(\bm{x},\bm{x})}\leq\kappa$, ensuring that the eigenvalues of the Gram matrix $\kkernelm$ are in $[0,\kappa^2]$.
\begin{definition}[Admissible regularizer]\label{filterdef}
A function $\filterfun_\regparam:[0,\kappa^2]\rightarrow \mathbb{R}$, parameterized by $0<\regparam\leq\kappa^2$, is an admissible regularizer if there exist constants $D, B,\gamma\in\mathbb{R}$ and $\alternative{\nu},\gamma_\nu > 0$ such that
\[
\sup_{0<\sigma\leq\kappa^2}\arrowvert\sigma\filterfun_\regparam(\sigma)\arrowvert\leq D\textnormal{, }
\sup_{0<\sigma\leq\kappa^2}\arrowvert\filterfun_\regparam(\sigma)\arrowvert\leq\frac{B}{\regparam}\textnormal{, }
\sup_{0<\sigma\leq\kappa^2}\arrowvert 1-\sigma\filterfun_\regparam(\sigma)\arrowvert\leq \gamma\,,
\]
\[
\text{and }\sup_{0<\sigma\leq\kappa^2} \frac{\regparam^\nu}{\sigma^\nu}\arrowvert 1-\sigma\filterfun_\regparam(\sigma)\arrowvert\leq \gamma_\nu,\, \quad\text{for any }\nu\in\,]0,\alternative{\nu}]\,,
\]
where the constant $\gamma_\nu$ does not depend on $\regparam$.
\end{definition}
The constant $\alternative{\nu}$ is in the literature called the qualification of the regularizer and it is related to the consistency properties of the learning method as described in more detail below.

The spectral filter is a matrix function that acts as a stabilized generalization of a matrix inverse. Hence, Eq.~(\ref{illposedproblem}) can be solved by
\eq{
\bm{\boldsymbol\alpha}&=\filterfun_\regparam(\kkernelm)\labelvec\\
&=\bm{W}\filterfun_\regparam(\bm{\Lambda})\bm{W}\transpose\ve(\bm{Y})\,.
}
Similarly, the noisy measurements can be filtered to obtain a better estimation of the true labels:
\eq{
\bm{\predfun} &= \kkernelm\bm{\boldsymbol\alpha}\nonumber\\
&= \bm{W}\bm{\Lambda}\bm{W}\transpose\bm{W}\filterfun_\regparam(\bm{\Lambda})\bm{W}\transpose\ve(\bm{Y}) \nonumber \\
&= \bm{W}\bm{\Lambda}\filterfun_\regparam(\bm{\Lambda})\bm{W}\transpose\ve(\bm{Y})\,. \nonumber
}
The spectral interpretation allows for using a more general form of the hat matrix (Eq.~(\ref{hatk})):
\eq{\bm{H}^\kkernelf=\bm{W}\bm{\Lambda}\filterfun_\regparam(\bm{\Lambda})\bm{W}\transpose\,.
}
For example, the filter function corresponding to the Tikhonov regularization, as used for independent-task \gls{krr}, is given by
\eq{
\filterfun^\mathrm{TIK}_\regparam(\bm{\sigma})=\frac{1}{\bm{\sigma}+\regparam}\,,
}
with the ordinary least-squares approach corresponding to $\regparam=0$. Several other learning approaches, such as spectral cut-off, iterated Tikhonov and $L2$ Boosting, can also be expressed as filter functions, but cannot be expressed as a penalized empirical error minimization problem analogous to Eq.~(\ref{tikhonov})~\citep{LoGerfo2008}. The spectral interpretation can also be used to motivate novel learning algorithms.

Many authors have expanded this framework to multi-task settings -- e.g.~\citet{Baldassarre2012,Argyriou2007,Argyriou2010}. We translate the pairwise learning methods from Section~\ref{dyadicprediciton} to this spectral regularization context. Let us denote the eigenvalue decomposition of the instance and task kernel matrices as
\eq{
\bm{\dkernelm} = \bm{U} \bm{\Sigma} \bm{U}\transpose \qquad\text{and}\qquad \bm{\tkernelm} = \bm{V} \bm{S} \bm{V}\transpose\,.
}
Let $\bm{u}_i$ denote the $i$-th eigenvector of $\bm{\dkernelm}$ and $\bm{v}_j$ the $j$-th eigenvector of $\bm{\tkernelm}$. The eigenvalues of the kernel matrix obtained with the Kronecker product pairwise kernel on a complete training set can be expressed as the Kronecker product $\bm{\Lambda}=\bm{S}\otimes\bm{\Sigma}$ of the eigenvalues $\bm{\Sigma}$ and $\bm{S}$ of the instance and task kernel matrices. For the models in this paper, it is opportune to define a pairwise filter function over the representation of the instances and tasks.

\newcommand{\ebound}{a}
\newcommand{\rbound}{b}
Both of the factor kernels are assumed to be bounded, and hence we can write that all the eigenvalues $\varsigma$ of the Kronecker product kernel can be factorized as the product of the eigenvalues of the instance and task kernels as follows:
\eqn{\label{evalfactorization}
\varsigma=\sigma s\qquad\textnormal{with }0\leq\sigma, s\leq\ebound\sqrt{\varsigma}\textnormal{ and }1\leq\ebound<\infty\,,
}
where $\sigma,s$ denote the eigenvalues of the factor kernels and $\ebound$ the constant determined as the product of $\sup_{\objectu\in\Uspace}\sqrt{\kernelf(\objectu,\objectu)}$ and  $\sup_{\objectv\in\Vspace}\sqrt{\gkernelf(\objectv,\objectv)}$.
\begin{definition}[Pairwise spectral filter]
We say that a function $\filterfun_\regparam:[0,\kappa^2]\rightarrow \mathbb{R}$, parameterized by $0<\regparam\leq\kappa^2$, is a pairwise spectral filter if it can be written as
\[
\filterfun_\regparam(\varsigma)=\vartheta_\regparam(\sigma,s)
\]
for some function $\vartheta_{\regparam}:[0,\ebound\sqrt{\varsigma}]^2\rightarrow \mathbb{R}$ with $1\leq\ebound<\infty$, and it is an admissible regularizer for all possible factorizations of the eigenvalues as in Eq.~(\ref{evalfactorization}).
\end{definition}

Since the eigenvalues of a Kronecker product of two matrices are just the scalar product of the eigenvalues of the matrices, the filter function for Kronecker \gls{krr} is given by
\eqn{\label{kronfilter}
\vartheta_\regparam^\mathrm{KK}(s, \sigma) = \filterfun_\lambda^\mathrm{TIK}(\sigma s)=\frac{1}{(\sigma s+\regparam)} \,,
}
where $\sigma$ and $s$ are the eigenvalues of $\bm{K}$ and $\bm{G}$, respectively. The admissibility of this filter is a well-known result, since it is simply the Tikhonov regularizer for the pairwise Kronecker product kernel.

Instead of considering two-step \gls{krr} from the kernel point of view, one can also cast it into the spectral filtering regularization framework. We start from Eq.~(\ref{partwostepmatrix}) in vectorized form:
\eq{
\ve(\bm{A}) &= \left((\bm{\tkernelm} + \lambda_\objectv\idmatrix)^{-1} \otimes (\bm{\dkernelm} + \lambda_\objectu\idmatrix)^{-1}\right) \ve(\bm{Y})\\
 &= \left(( \bm{V} \bm{S} \bm{V}\transpose+ \lambda_\objectv\idmatrix)^{-1} \otimes (\bm{U} \bm{\Sigma} \bm{U}\transpose+ \lambda_\objectu\idmatrix)^{-1}\right) \ve(\bm{Y})\\
& =\left(( \bm{V} \filterfun^\mathrm{TIK}_{\regparam_\objectv}(\bm{S}) \bm{V}\transpose) \otimes (\bm{U} \filterfun^\mathrm{TIK}_{\regparam_\objectu}(\bm{\Sigma})  \bm{U}\transpose)\right) \ve(\bm{Y})\\
& =\left((  \bm{V}\otimes \bm{U}) (\filterfun^\mathrm{TIK}_{\regparam_\objectv}(\bm{S})\otimes \filterfun^\mathrm{TIK}_{\regparam_\objectu}(\bm{\Sigma}) ) (  \bm{V}\otimes \bm{U})\transpose\right) \ve(\bm{Y})\,.
}
Hence, one can interpret two-step \gls{krr} with a complete training set for Setting~D as a spectral filtering regularization-based learning algorithm that uses the pairwise Kronecker product kernel with the following filter function:
\eqn{
\vartheta^\mathrm{TS}_{\regparam}(s,\sigma)&=\filterfun^\mathrm{TIK}_{\regparam_\objectv}(s)\filterfun^\mathrm{TIK}_{\regparam_\objectu}(\sigma)\nonumber\\
& =\frac{1}{(\sigma+\regparam_\objectu)(s+\regparam_\objectv)}\nonumber\\
&=\frac{1}{\sigma s+\regparam_\objectv\sigma+\regparam_\objectu s+\regparam_\objectv\regparam_\objectu}\,.\label{twostepfilter}
}
The validity of this filter is characterized by the following theorem.
\begin{theorem}
The filter function $\vartheta^\mathrm{TS}_{\regparam}(\cdot,\cdot)$ is admissible with $D=B=\gamma=1$, $\gamma_\nu=2\ebound\rbound$, and has qualification $\alternative{\nu}=\frac{1}{2}$ for all factorizations of $\varsigma$ and $\regparam$ as
\eqn{\label{factorizations}
\varsigma=\sigma s\textnormal{ and }\regparam=\regparam_\objectv\regparam_\objectu\quad\text{ with }0\leq\sigma, s\leq\ebound\sqrt{\varsigma}\textnormal{ and }0<\regparam_\objectv,\regparam_\objectu\leq\rbound\sqrt{\regparam}\,,
}
where $1\leq\ebound,\rbound<\infty$ are constants that do not depend on $\regparam$ or $\varsigma$.
\end{theorem}
\begin{proof}
Let us recollect the last condition in Definition~\ref{filterdef}:
\[
\sup_{0<\varsigma\leq\kappa^2}\frac{\varsigma^\nu}{\regparam^\nu}\arrowvert 1-\varsigma\filterfun_\regparam(\varsigma)\arrowvert\leq \gamma_\nu,\quad\text{for any }\nu\in\,]0,\alternative{\nu}]\,,
\]
where $\gamma_\nu$ does not depend on $\regparam$. In order to show this for all cases covered by Eq.~(\ref{factorizations}), we rewrite the condition by taking the supremum with respect to the factorizations of $\varsigma$ and $\regparam$ given the constants $\ebound$ and~$\rbound$:
\eq{
\sup_{
\underset{\underset{ 0<\sigma, s\leq\ebound\sqrt{\varsigma}}{ 0<\regparam_\objectv,\regparam_\objectu\leq\rbound\sqrt{\regparam}}}{0<\varsigma\leq\kappa^2}}\frac{\varsigma^\nu}{\regparam^\nu}\left(1- \frac{\varsigma}{\varsigma+\regparam_\objectv\sigma+\regparam_\objectu s+\regparam}\right)\leq \gamma_\nu,\qquad\text{ for any }\nu\in\,]0,\alternative{\nu}]\,.
}
The left-hand side then becomes
\eq{
\sup_{0<\varsigma\leq\kappa^2}\frac{\varsigma^\nu}{\regparam^\nu}\left(1- \frac{\varsigma}{\varsigma+2\ebound\rbound\sqrt{\regparam}\sqrt{\varsigma}+\regparam}\right)
=\sup_{0<\varsigma\leq\kappa^2}\left(\frac{2\ebound\rbound\regparam^{\frac{1}{2}-\nu}\varsigma^{\nu+\frac{1}{2}}+\regparam^{1-\nu}\varsigma^\nu}{\varsigma+2\ebound\rbound\sqrt{\regparam}\sqrt{\varsigma}+\regparam}\right)\,.
}
By checking the extreme values of the latter expression with respect to $(\varsigma,\regparam,\nu)$ using standard differential calculus, we observe that it is bounded by $\gamma_\nu=2\ebound\rbound$ if $\nu\in\,]0,\frac{1}{2}]$. With values of $\alternative{\nu}$ larger than~$\frac{1}{2}$, the term $2\ebound\rbound\regparam^{\frac{1}{2}-\nu}\varsigma^{\nu+\frac{1}{2}}$ in the numerator grows arbitrarily while $\regparam\rightarrow 0$, and hence the qualification is $\alternative{\nu}=\frac{1}{2}$. The other conditions in Definition~\ref{filterdef} can be verified by direct computation.
\end{proof}
Thus, Eq.~(\ref{twostepfilter}) can be positioned within the spectral filtering regularization-based framework with separate regularization parameter values for instances and tasks. In contrast to Eq.~(\ref{kronfilter}), the filter of two-step \gls{krr} can be factorized into a component for the tasks and instances separately:
 \eqn{
\vartheta_\regparam(s, \sigma)=\filterfun_{\regparam_\objectu}(\sigma)\filterfun_{\regparam_\objectv}(s)\label{decompfilter}\,.
}

Providing a different regularization for instances and tasks also makes sense from a learning point of view. It is easy to imagine a setting in which the instance has a much larger influence in determining the label compared to the task or vice versa. For example, consider a collaborative filtering setting with the goal of recommending books for customers. Suppose that the sales of a book are for a very large part determined simply by being a bestseller novel or not, and less by individual customer's taste. When building a predictive model, one would give more freedom to the part concerning the books (hence a lower regularization). Less degrees of freedom are given to the inference of the user's personal task, as this is harder to learn and explains less of the variance in the preferences. This can be extended even further, by choosing specific filter functions separately for the instances and tasks tuned to the application at hand. In a pairwise setting, the filter function to perform independent-task \gls{krr} arises as a special case with $\lambda_\objectv=0$:
\eq{
\vartheta^\mathrm{IT}_{\regparam_\objectv}(s,\sigma) = \frac{1}{(\sigma+\lambda_\objectu)s}\,,
}
when the task kernel is full rank (see Theorem~\ref{ITKRRTSeqiuivalence}).

Next, we analyze the consistency properties of two-step \gls{krr} in Setting~D, given the above results about the universality of the pairwise Kronecker product kernel and the spectral filtering interpretation of the method.
Let $R(\cdot)$ denote the expected prediction error of a hypothesis $\predfun$ with respect to some unknown probability measure $\rho(\bm{x}, y)$ on the joint space $\ispace\times\mathbb{R}$ of inputs and labels, that is,
\[
R(\predfun) = \int_{\ispace\times\mathbb{R}}(\predfun(\bm{x})- y)^2 \text{d}\rho(\bm{x}, y)\,.
\]
Given the input space $\ispace$, the minimizer of the error is the so-called regression function:
\[
\predfun_\rho(\bm{x})=\int_{\mathbb{R}} y\ \text{d}\rho(y\mid\bm{x})\,.
\]
Following \citet{Baldassarre2012,LoGerfo2008,Bauer2007regularization}, we characterize the quality of a learning algorithm via its consistency properties. In particular, by considering whether the learning algorithm is consistent in the sense of Definition~\ref{consistencydef}.
\begin{definition}
\label{consistencydef}
A learning algorithm is consistent if the following holds with high probability
\[
\lim_{\lsize\rightarrow\infty}
\int_{\ispace}\left(\hat{\predfun}^\regparam_\lsize(\bm{x}) - \predfun_\rho(\bm{x})\right)^2 \text{d}\rho(\bm{x})=0\,,
\]
where $\hat{\predfun}^\regparam_\lsize$ denotes the hypothesis inferred by the learning algorithm from a training set having $\lsize$ independently and identically drawn training examples.
\end{definition}

The following result is assembled from the existing literature concerning spectral filtering based regularization methods and we present it here only in a rather abstract form. For the exact details and further elaboration, we refer to \citet{Baldassarre2012,LoGerfo2008,Bauer2007regularization}.
\begin{theorem}
If the filter function is admissible and the kernel function is universal, then the learning algorithm is consistent in the sense of Def.~\ref{consistencydef}. Furthermore, if the regularization parameter is set as $\regparam=\frac{1}{\lsize^{2\alternative{\nu}+1}}$, where $\lsize$ denotes the number of independently and identically drawn training examples, then the following holds with high probability:
\eqn{
R(\hat{\predfun}^\regparam)-R(\predfun_\rho(\bm{x}))
 = \mathcal{O}\left(\lsize^{-\frac{\alternative{\nu}}{2\alternative{\nu}+1}}\right)\,.
}
\end{theorem}
Intuitively put, the universality of the kernel ensures that the regression function belongs to the hypothesis space of the learning algorithm
and the admissibility of the regularizer ensures that $R(\hat{\predfun}^\regparam)$ converges to it when the size of the training set approaches infinity and the rate of convergence is reasonable.
\begin{mycol}
Two-step \gls{krr} is consistent and the hypothesis it infers from the training set of size $\lsize=\osize\qsize$ converges to the underlying regression function with a rate at least proportional to
\eqn{
R(\hat{\predfun}^\regparam)-R(\predfun_\rho(\objectu, \objectv))
 = \mathcal{O}\left(\min(\osize,\qsize)^{-\frac{\alternative{\nu}}{2\alternative{\nu}+1}}\right)\,.
}
\end{mycol}
\begin{proof}
The result follows from the admissibility of the pairwise filter function, the universality of the pairwise Kronecker product kernel and the fact that the training set consists of at least $\min(\osize,\qsize)$ independently and identically drawn training examples.
\end{proof}
Hence, it is proven that the two-step \gls{krr} is not only a universal method (can approximate any pairwise prediction function), but will also converge to the prediction function that generated the data when provided with enough training examples.

\section{Related work}

In the introduction we argued that it remains important to study the theoretical properties of kernel methods for three reasons: (a) kernel methods are general-purpose instruments, (b) they often serve as bulding blocks for more complicated methods, and (c) they clearly outperform other methods for specific scenarios such as cross-validation. As such observations have been reported in other papers, including quantitative results on real-world datasets, we see no merit in providing additional experimental evidence. We refer to other works that pairwisely compare the kernel methods discussed in this article with other machine learning methods -- e.g.~\citet{Ding2013,Romera-paredes2015,Schrynemackers2015,Stock2016a}. However, it remains important to outline the commonalities and differences with other methods. In what follows, we subdivide these methods according to their applicability to Settings A, B, C or D.

\subsection{Methods that are applicable to Setting A}
\label{sec:matrix-completion-and-hybrid-methods}

In this section, we review methods for Setting A, i.e.\ matrix completion methods. In Section~2, such methods were claimed to be useful for a pairwise learning setting with partially-observed matrices $\mathbf{Y}$. Both $\objectu$ and $\objectv$ are observed, but not for all instance-target combinations. In Setting A, side information about instances or targets is not required per se. We hence distinguish between methods that ignore side information and methods that also exploit such information, in addition to analyzing the matrix $\mathbf{Y}$. 

Inspired by the Netflix challenge in 2006, the former type of methods has been mainly popular in the area of recommender systems. Those methods often impute missing values by computing a low-rank approximation of the sparsely-filled matrix $\mathbf{Y}$, and many variants exist in the literature, including algorithms based on nuclear norm minimization \citep{Candes2008}, Gaussian processes \citep{Lawrence2009}, probabilistic methods \citep{Shan2010}, spectral regularization \citep{Mazumder2010}, non-negative matrix factorization \citep{Gaujoux2010}, graph-regularized non-negative matrix factorization~\citep{Cai2011} and alternating least-squares minimization \citep{Jain2013}. In addition to recommender systems, matrix factorization methods are commonly applied to social network analysis \citep{Menon2010}, biological network inference \citep{Gonen2012,Liu2015}, and travel time estimation in car navigation systems~\citep{Dembczynski_et_al_2013}.

In addition to matrix factorization, a few other methods exist for Setting A. Historically, memory-based collaborative filtering has been popular, and corresponding methods are very easy to implement. They make predictions for the unknown cells of the matrix by modelling a similarity measure between either rows or columns -- see e.g.\ \citep{Takacs2008}. For example, when rows and columns correspond to users and items, respectively, then one can predict novel items for a particular user by searching for other users with similar interests. To this end, different similarity measures are commonly used, including cosine similarity, Tanimoto similarity and statistical similarity measures \citep{CorinaBasnouPalomaVicente2015a}. 

Many variants of matrix factorization and other collaborative methods have been presented, in which side information of rows and columns is considered during learning, in addition to exploiting the structure of the matrix $\mathbf{Y}$ -- see e.g.\ \citep{Basilico2004,Abernethy2008,Adams2010,Fang2011,Zhou2011a,Menon2011,Zhou2012a,Gonen2012,Liu2015a,Ezzat2017}. One simple but effective method is to extract latent feature representations for instances and targets in a first step, and combine those latent features with explicit features in a second step \citep{Volkovs2012}. To this end, the methods that have been described in this article could be used, as well as other pairwise learning methods that depart from explicit feature representations.  


\subsection{Methods that are applicable to Settings B and C}

When side information is available for the objects $\objectu$ and $\objectv$, it would be pointless to ignore this information. The hybrid filtering methods from the previous paragraph seek to combine the best of both worlds, by simultaneously modeling side information and the structure of $\mathbf{Y}$. In addition to Setting A, they can often be applied to Settings B and C, which coincide, respectively, with a novel user and a novel item in recommender systems. In that context, one often speaks of cold-start recommendations.  

However, when focusing on Settings B and C only, a large bunch of machine learning methods is closely connected to pairwise learning. In fact, many multi-target prediction problems can be interpreted as specific pairwise learning problems. All multi-task learning problems, and multi-label classification and multivariate regression problems as special cases, can be seen as pairwise learning problems, by calling $\objectu$ an instance and $\objectv$ a label (a.k.a. target/output/task). We refer the reader to \citet{Waegeman2018} for a recent review on connections between multi-target prediction problems and pairwise learning. 

Multi-task learning, multi-label classification and multivariate regression are huge research fields, so it is beyond the scope of this paper to give an in-depth review of all methods developed in those fields. Moreover, not all multi-target prediction methods are relevant for the discussion we intend to provide. Roughly speaking, simple multi-target prediction methods only consider side information for one type of objects, let's say the objects $\objectu$, which represent the instances. No side information is available for the targets, which could then be denoted $\objectv$. Since no side information is available for the targets, simple multi-target prediction methods can only be applied to Setting B and C. Remark that $\objectu$ and $\objectv$ are interchangeable, so Settings B and C are identical settings from a theoretical point of view.

The situation changes when side information in the form of relations or feature representations becomes available for both instances and targets. In such a scenario, multi-target prediction methods that process side information about targets are more closely related to the pairwise learning methods that are analyzed in this article. We will therefore provide a thorough review of such methods in the next paragraph. Furthermore, remark that the availability of side information on both instance and target level implies that now also Setting D can be covered, in addition to Settings B and C. So, exploiting side information about targets has two main purposes: it might boost the predictive performance in Settings B and C, and it is essential for generalizing to novel targets in Setting~D.    

\subsection{Methods that are applicable to Settings B, C and D}
\label{sec:methods-that-use-target-representations}


In Setting D, side information for both $\objectu$ and $\objectv$ is essential for generalizing to zero-shot problems, such as a novel target molecule in drug discovery, a novel tag in document tagging, or a novel person in image classification. In this area, kernel methods have played a prominent role in the past, but also tree-based methods are commonly used \citep{Geurts2007,Schrynemackers2015}. In bioinformatics a subdivision is usually made between global methods, which construct one predictive model, and local methods, which separate the problem into several subproblems \citep{Vert2008,Bleakley2009a,Schrynemackers2013}.

Factorization machines~\citep{Rendle2010,Steffen2012} deserve a special mention here, as they can be seen as an extension of matrix factorization methods towards Settings B, C and D. They work by simultaneously learning a lower-dimensional feature embedding and a polynomial (usually of degree two) predictive model. Factorization machines can effectively cope with large, sparse data sets frequently encountered in collaborative and content-based filtering. For such problems they are expected to outperform kernel methods. Their main drawback, however, is that training them becomes a non-convex problem and requires relatively large data sets to train. The relation between factorization machines, polynomial networks and kernel machines was recently explored by~\citet{Blondel2016}.

In recent years, specific zero-shot learning methods based on deep learning have become extremely popular in image classification applications. The central idea in all those methods is to construct semantic feature representations for class labels, for which various techniques might work. One class of methods constructs binary vectors of visual attributes \citep{Lampert2009,Palatucci2009,Liu2011,Fu2013}. Another class of methods rather considers continuous word vectors that describe the linguistic context of images \citep{Mikolov2013,Frome2013,Socher2013}.

Many zero-shot learning methods for image classification adopt principles that originate from kernel methods. The model structure can often be formalized as follows: 
\begin{equation}
\label{eq:pairwise}
f(\objectu,\objectv) = \mathbf{w}^T \big(\phi(\objectu) \otimes \psi(\objectv)\big)    
\end{equation}
with $\mathbf{w}$ a parameter vector and $\phi$ an embedding of an object in a high-dimensional feature space. This model in fact coincides with the primal formulation of Eq.~(6) with $\Gamma$ the Kronecker product pairwise kernel. Different optimization problems with this model have been proposed \citep{Frome2013,Akata2015,Akata2016}, and related methods provide nonlinear extensions \citep{Socher2013,Xian2016b}. Most of these optimization problems do not minimize squared error loss, and they should rather be seen as structured output prediction methods. Indeed, a representation such as (\ref{eq:pairwise}) is in fact commonly used in structured output prediction methods. These methods additionally have inference procedures that allow for finding the best-scoring targets in an efficient manner.

Some of the zero-shot learning methods from computer vision also turn out to be useful for the related field of text classification. For documents, it is natural to model a latent representation for both the (document) instances and class labels in a joint space \citep{Nam2016}. Nonetheless, many of those approaches are tailor-made for particular application domains. In contrast to kernel methods, they do not provide general-purpose tools for analyzing general data types.

\section{Conclusions}
In this work we have studied several models derived from kernel ridge regression. First, we independently derived single-task kernel ridge regression, Kronecker kernel ridge regression, two-step kernel ridge regression and the linear filter. Subsequently, we have shown that they are all related; two-step kernel ridge regression and the linear filter are a special case of pairwise kernel ridge regression, itself being merely kernel ridge regression with a specific pairwise kernel. From this, universality and consistency results could be derived, motivating the general use of these methods.

Pairwise learning is a broadly applicable machine learning paradigm. It can be applied to problems as diverse as multi-task learning, content and collaborative filtering, transfer learning, network inference and zero-shot learning. This work offers a general toolkit to tackle such problems. Despite being easy to implement and computationally efficient, kernel methods have been found to attain an excellent performance on a wide variety of problems. As such, we believe that the intriguing algebraic properties of the Kronecker product will serve as a basis for developing novel learning algorithms, and we hope that the results of this work will be helpful in that regard. 

\section*{Acknowledgements}
Michiel Stock is supported by the Research Foundation - Flanders (FWO17/PDO/067). This work was supported by the Academy of Finland (grants 311273 and 313266 to Tapio Pahikkala and grant 289903 to Antti Airola).

\bibliographystyle{apalike}

%
%
%
%

\section*{Appendix: Matrix properties}
The trick of pairwise learning is transforming a matrix in a vector. This can be done by the vectorization operation.
\begin{definition}[Vectorization]
The vectorization operator $\ve(\cdot)$ is a linear operator that transforms an $n\times m$ matrix $\mathbf{A}$ in a column vector of length $nm$ by stacking the columns of $\mathbf{A}$ on top of each other.
\end{definition}
Further, the Kronecker product is defined as follows.
\begin{definition}[Kronecker product]
If $\mathbf{A}=[a_{ij}]$ is an $n \times m$ matrix and $\mathbf{B}=[ij]$ is an $p \times q$ matrix, then the Kronecker product $\mathbf{A} \otimes \mathbf{B}$ is the $mp \times nq$ block matrix:
\eq{
\mathbf{A}\otimes \mathbf{B} =
\begin{bmatrix}
          a_{11}\mathbf{B} & \dots &  a_{1m}\mathbf{B}\\
          \vdots & \ddots & \vdots\\
          a_{n1}\mathbf{B} & \dots &  a_{nm}\mathbf{B}
      \end{bmatrix}\,.
}
\end{definition}
For instance, if
\begin{equation*} \mathbf{A}=
      \begin{bmatrix}
          a_{11} & a_{12} \\ a_{21} & a_{22}
      \end{bmatrix}   \qquad \text{and}\qquad
\mathbf{B} = 
\begin{bmatrix}
          b_{11} & b_{12} \\ b_{21} & b_{22}
      \end{bmatrix}\,,
  \end{equation*}
then
$$
\ve(\mathbf{A}) = 
\begin{bmatrix}
          a_{11} \\ a_{12} \\ a_{21} \\ a_{22}
      \end{bmatrix}
 \qquad \text{and}\qquad
\mathbf{A}\otimes \mathbf{B} =
\begin{bmatrix}
          a_{11}b_{11} & a_{11}b_{12} & a_{12}b_{11} & a_{12}b_{12}\\
a_{11}b_{21} & a_{11}b_{22} & a_{12}b_{21} & a_{12}b_{22}\\
a_{21}b_{11} & a_{21}b_{12} & a_{22}b_{11} & a_{22}b_{12}\\
a_{21}b_{21} & a_{21}b_{22} & a_{22}b_{21} & a_{22}b_{22}
      \end{bmatrix}
$$
The relation between vectorization and the Kronecker product is given by the following property.
\begin{property}\label{pr:vectkron}
For any conformable matrices $\mathbf{N}, \mathbf{M}$ and $\mathbf{X}$, it holds that
\eq{
 (\mathbf{N}\transpose\otimes\mathbf{M}) \ve(\mathbf{X})=\ve(\mathbf{M}\mathbf{X}\mathbf{N})\,.
}
\end{property}
Computing the Kronecker product of two reasonably large matrices results in a huge matrix, often too large to fit in computer memory. If the Kronecker product is only needed in an intermediary step, the above identity can be used to dramatically reduce computation time and memory requirement.

Using the eigenvalue decomposition of matrices, a large system of equations using the Kronecker product can be solved efficiently.
\begin{property}\citep{pahikkala2013conditional}\label{shiftedKronprop}
Let $\mathbf{A}, \mathbf{B} \in \realset^{n\times n}$ be diagonalizable matrices, i.e.~matrices that can be eigen decomposed as
\eq{
\mathbf{A} = \mathbf{V}\boldsymbol{\Lambda}\mathbf{V}^{-1}\text{ and } \mathbf{B} = \mathbf{U}\boldsymbol{\Sigma}\mathbf{U}^{-1}\,,
}
where $\mathbf{V}, \mathbf{U} \in \realset^{n\times n}$ contain the eigenvectors and the diagonal matrices $\boldsymbol{\Lambda}, \boldsymbol{\Sigma} \in \realset^{n\times n}$ contain the corresponding eigenvalues of $\mathbf{A}$ and $\mathbf{B}$. Then, the following type of shifted Kronecker product system
\eqn{
(\mathbf{A} \otimes \mathbf{B} + \lambda \idmatrix)\mathbf{a} = \ve(\labelmatrix)\,,\label{eq:shiftedKronSys}
}
 where  $\lambda> 0$ and $\labelmatrix\in \realset^{n\times n}$, can be solved with respect to $\mathbf{a}$ in $\mathcal{O}(n^3)$ time if the inverse of $(\mathbf{A} \otimes \mathbf{B} + \lambda \idmatrix)$ exists.
\end{property}
\begin{proof}
By multiplying both sides of Eq.~(\ref{eq:shiftedKronSys}) by $(\mathbf{A} \otimes \mathbf{B} + \lambda \idmatrix)^{-1}$, it is relatively straightforward to show that
\eqn{
\mathbf{a} = \ve(\mathbf{V}(\mathbf{C} \odot \mathbf{E})(\mathbf{U}\transpose)^{-1})\,,\label{eq:complexitislate}
}
with $\odot$ the Hadamard product (element-wise matrix multiplication),
\eq{
\mathbf{E} = \mathbf{U}^{-1}\labelmatrix(\mathbf{V}^{-1})\transpose
}
and
\eq{
\diagm(\ve(\mathbf{C})) = (\boldsymbol{\Lambda} \otimes \boldsymbol{\Sigma} + \lambda \idmatrix)^{-1}\,.
}
The eigen decompositions of $\mathbf{A}$ and $\mathbf{B}$ as well as all matrix multiplications in Eq.~(\ref{eq:complexitislate}) can be computed in $\mathcal{O}(n^3)$ time.
\end{proof}
Lastly, we present two matrix identities that are useful in deriving the linear filter of Section~\ref{smoothertofilter}. Consider two matrices of the form
\begin{equation*}
\mathbf{A} = a_1 \idmatrix_\osize\otimes\idmatrix_\qsize + a_2 \ones_\osize\otimes\idmatrix_\qsize + a_3 \idmatrix_\osize\otimes\ones_\qsize + a_4 \ones_\osize\otimes\ones_\qsize 
\end{equation*}
and 
\begin{equation*}
\mathbf{B} = b_1 \idmatrix_\osize\otimes\idmatrix_\qsize + b_2 \ones_\osize\otimes\idmatrix_\qsize + b_3 \idmatrix_\osize\otimes\ones_\qsize + b_4 \ones_\osize\otimes\ones_\qsize\,.
\end{equation*}
Two properties can easily be deduced.
\begin{property}\label{pr:smoothmult}
The product $\mathbf{C}=\mathbf{A}\mathbf{B}$ is given by
\begin{equation*}
\mathbf{C} = c_1 \idmatrix_\osize\otimes\idmatrix_\qsize + c_2 \ones_\osize\otimes\idmatrix_\qsize + c_3 \idmatrix_\osize\otimes\ones_\qsize + c_4 \ones_\osize\otimes\ones_\qsize\,,
\end{equation*}
with
\begin{align*}
c_1& = a_1b_1\\
c_2 &= a_1b_2 + a_2b_1+a_2b_2\osize\\
c_3 &= a_1b_3 + a_3b_1+a_3b_3\qsize\\
c_4 &= a_1b_4 +a_2b_3+a_2b_4\osize+a_3b_2+a_3b_4\qsize + a_4b_1+a_4b_2\osize+a_4b_3\qsize+a_4b_4\osize\qsize\,.
\end{align*}
\end{property}
\begin{property}\label{pr:smoothinv}
The inverse $\mathbf{D}=\mathbf{A}^{-1}$ is given by
\begin{equation*}
\mathbf{D} = d_1 \idmatrix_\osize\otimes\idmatrix_\qsize + d_2 \ones_\osize\otimes\idmatrix_\qsize + d_3 \idmatrix_\osize\otimes\ones_\qsize + d_4 \ones_\osize\otimes\ones_\qsize\,,
\end{equation*}
with
\begin{align*}
d_1& = \frac{1}{a_1}\\
d_2 &=  \frac{-a_2}{a_1(a_1+a_2\osize)} \\
d_3 &= \frac{-a_3}{a_1(a_1+a_3\qsize)}\\
d_4 &=(a_2(a_1+a_3\qsize)(a_3+a_4\osize) +a_3(a_1+a_2\osize)(a_2+a_3\qsize +a_4\qsize)\\
& \quad -a_4(a_1+a_2\osize)(a_1+a_3\qsize))
(a_2(a_1+a_2\osize)(a_1+a_3\qsize)(a_1+a_2\osize\\
&\quad+a_3\qsize+a_4\osize\qsize))^{-1}\,.
\end{align*}
\end{property}

\end{document}